\documentclass[twoside]{article}

\usepackage[accepted]{aistats2020}
\usepackage[utf8]{inputenc} 
\usepackage[T1]{fontenc}    
\usepackage{hyperref}       
\usepackage{url}            
\usepackage{booktabs}       
\usepackage{amsfonts}       
\usepackage{nicefrac}       
\usepackage{microtype}      
\usepackage{amsmath,amsthm,amssymb}

\usepackage{subfigure} 
\usepackage{natbib}
\usepackage{comment,color}
\usepackage{bm}
\usepackage{wrapfig}
\usepackage{adjustbox}
\usepackage[ruled]{algorithm2e} 

\newcommand{\F}{{\mathcal{F}}}
\newcommand{\X}{{\mathcal{X}}}

\renewcommand{\H}{{\mathcal{H}}}
\newcommand{\E}{{\mathbb{E}}}

\newcommand{\R}{\mathbb{R}}
\newcommand{\Rn}{{\mathbb{R}^n}}

\newcommand{\y}{{\bm y}}

\newcommand{\ssf}{\mathsf{f}}

\newcommand{\ssg}{\mathsf{g}}

\newcommand{\supp}{\mathrm{supp}}

\newtheorem{theorem}{Theorem}
\newtheorem{lemma}{Lemma}

\newtheorem{assumption}{Assumption}

\newcommand{\argmax}{\operatornamewithlimits{argmax}}
\newcommand{\argmin}{\operatornamewithlimits{argmin}}

\begin{document}


\twocolumn[

\aistatstitle{Simulator Calibration under Covariate Shift with Kernels}
\aistatsauthor{ Keiichi Kisamori \And Motonobu Kanagawa \And  Keisuke Yamazaki}
\aistatsaddress{ NEC and AIST, Japan
\And  EURECOM, France 
\And AIST, Japan}
]

\begin{abstract}


We propose a novel calibration method for computer simulators, dealing with the problem of covariate shift.
Covariate shift is the situation where input distributions for training and test are different, and ubiquitous in applications of simulations. 
Our approach is based on Bayesian inference with kernel mean embedding of distributions, and on the use of an importance-weighted reproducing kernel for covariate shift adaptation.
We provide a theoretical analysis for the proposed method, including a novel theoretical result for conditional mean embedding, as well as empirical investigations suggesting its effectiveness in practice.
The experiments include calibration of a widely used simulator for industrial manufacturing processes, where we also demonstrate how the proposed method may be useful for sensitivity analysis of model parameters.
\end{abstract}

\section{Introduction}
\label{sec:Introduction}

\begin{figure}[t]
  \centering
  \includegraphics[width=\linewidth]{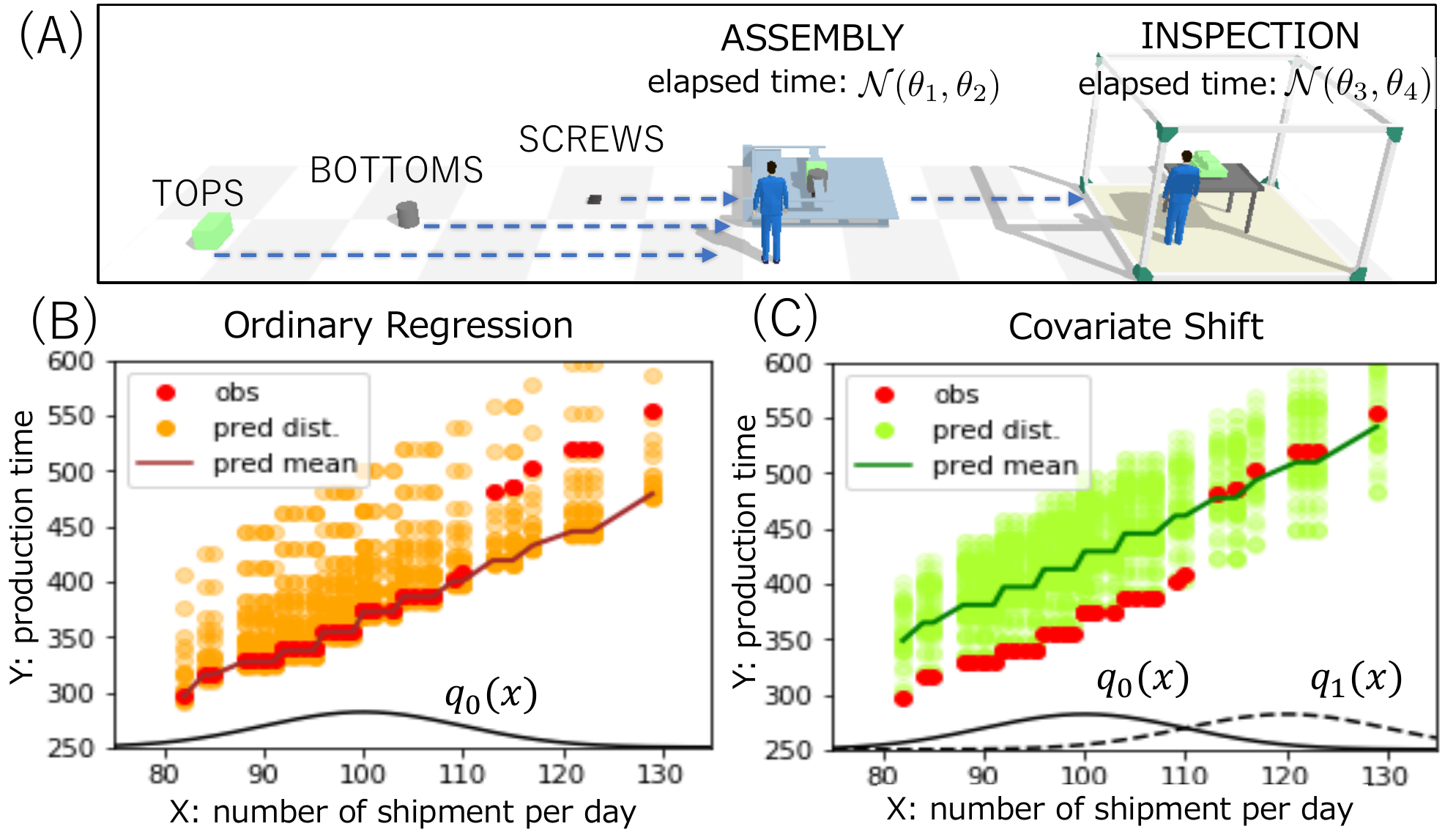}
  \caption{(A) Illustration of a manufacturing process simulator for assembling products. In the factory, one product is made from three items (TOPS, BOTTOMS and SCREWS) by the ASSEMBLY machine, and four such products are checked by the INSPECTION machine at the same time.
Parameter $\theta$ of the simulation model $r(x,\theta)$ consists of 4 constants:  mean $\theta_1$ and variance $\theta_2$ of the distribution of the processing time in the ASSEMBLY machine, and those (described as $\theta_3$ and $\theta_4$) in the INSPECTION machine.
  (B) Results of our method {\em without} covariate shift adaptation: training data (red points), generated predictive outputs (orange) and their means (brown curve).
  (C) Results of our method {\em with} covariate shift adaptation: training data (red points), generated predictive outputs (light green) and their means (green curve). 
  $q_0(x)$ and $q_1(x)$ are input densities for training and prediction, respectively. 
  More details in Secs.~\ref{sec:Introduction} and \ref{subsec:ExperimentProductSimulator}.}
  \label{fig:SimpleAssemblyModel}
\end{figure}

Computer simulators are ubiquitous in many areas of science and engineering, examples including climate science, social science, and epidemics, to just name a few \citep{winsberg2010science,weisberg2012simulation}.  
Such tools are useful in understanding and predicting complicated time-evolving phenomena of interest.
Computer simulators are also widely used in industrial manufacturing process modeling \citep{Mourtzis2014}, and we use one such simulator described in Fig.~\ref{fig:SimpleAssemblyModel}-(A), which models an assembling process of certain products in a factory, as our working example. 
 
In this work we deal with the task of {\em simulator calibration} \citep{KenHag01}, which is necessary to make simulation-based predictions reliable. 
To describe this, we introduce some notation used in the paper. 
We are interested in a system $R(x)$ that takes $x$ as an input and output $y = R(x) + \varepsilon$ possibly corrupted by a noise $\varepsilon$.
This system $R(x)$ is of interest but not known.
Instead, we are given data $(X_i,Y_i)_{i=1}^n$ from the system,  where input locations $X_1,\dots,X_n$ are generated from a distribution $q_0(x)$ and outputs $Y_1,\dots,Y_n$ from the target system $Y_i = R(X_i) + \varepsilon_i$. 
On the other hand, a simulator is defined as a function $r(x,\theta)$ that takes $x$ as an input and outputs $r(x,\theta)$, where $\theta$ is a model parameter. 
The task of simulator calibration is to tune (or estimate) the parameter $\theta$ so that the $r(x,\theta)$ ``approximates well'' the unknown target system $R(x)$ by using the data  $(X_i,Y_i)_{i=1}^n$.
For instance, in Fig.~\ref{fig:SimpleAssemblyModel}, the target system $R(x)$ takes as an input the number $x$ of required products to be manufactured in one day, and outputs the total time $y = R(x) + \varepsilon$ required for producing all the products; the simulator $r(x,\theta)$ models this process (see the ``pred mean'' curves in Fig.~\ref{fig:SimpleAssemblyModel}-(B)(C)).

There are mainly two challenges in the task of simulator calibration, which distinguish it from standard statistical learning problems. 
The first one owes to the complexity of the simulation model.
Very often, a simulation model $r(x,\theta)$ cannot be written as a simple function of the input $x$ and parameter $\theta$, because the process of producing the output $y=r(x,\theta)$ may involve various numerical algorithms (e.g., solutions for differential equations) and/or IF-ELSE type decision rules of multiple agents.
Therefore, one cannot access the gradient of the simulator output $r(x,\theta)$ with respect to the parameter $\theta$, and thus calibration cannot reply on gradient-based methods for optimization (e.g., gradient descent) and sampling (e.g., Hamiltonian Monte Carlo).
Moreover, one simulation $y = r(x,\theta)$ for a given input $x$ can be computationally very expensive. 
Thus only a limited number of simulations can be performed for calibration.
To summarise, the first challenge is that calibration should be done by only making use of forward simulations (or evaluations of $r(x,\theta)$), while the number of simulations cannot be large.

The second challenge is that of {\em covariate shift} (or {\em sample selection bias})  \citep{Shimodaira00,Sugiyama2012}, which is ubiquitous in applications of simulations, but has been rarely discussed in the literature on calibration methods.
The situation is that the input distribution $q_1(x)$ for the test (or prediction) phase is {\em different} from the input distribution $q_0(x)$ generating the training input locations $X_1,\dots,X_n$. 
In other words, the parameter $\theta$ is to be tuned so that the simulator $r(x,\theta)$ accurately approximates the target system $R(x)$ with respect to the distribution $q_1(x)$ (e.g., the error defined as $\int (R(x) - r(x,\theta))^2 q_1(x) dx$ is to be small), while training data $(X_i,Y_i)_{i=1}^n$ are only given with respect to another distribution $q_0(x)$.

The covariate shift setting is inherently important and ubiquitous in applications of computer simulation, because the purpose of a simulation is often in {\em extrapolation}.
An illustrative example is climate simulations, where the aim is to answer whether global warming will occur in the future. 
As such,  input $x$ is a time point and the target system $R(x)$ is the global temperature.
Calibration of the simulator $r(x,\theta)$ is to be done based on data from the past, but prediction is required for the future. 
This means that training input distribution $q_0(x)$ has a support in the past, but that of test $q_1(x)$ has a support on the future. 
For our working example in Fig.~\ref{fig:SimpleAssemblyModel}, training input locations $X_1,\dots,X_n$ from $q_0(x)$ are more densely distributed in the region $x<110$ than the region $x \geq 110$, since the data are obtained in a trial period.
On the other hand, the test phase (i.e., when the factory is deployed) is targeted on mass production, and thus the test input distribution $q_1(x)$ has mass concentrated in the region $x \geq 110$.

Being a parametric model, a simulator only has a finite degree of freedom, and thus cannot capture all the aspects of the target system. 
Under such a model misspecification, the covariate shift is known to have a huge effect: the optimal model for the test input distribution may be drastically different from that for the training input distribution \citep{Shimodaira00}. 
In climate simulations,  care must be taken in how to tune the simulator as the data are only from the past;  otherwise, the resulting predictions about the future will not be reliable \citep{winsberg2018philosophy}.
In the example of Fig.~\ref{fig:SimpleAssemblyModel},  the behavior of the target system $R(x)$ changes for the trial and test phases: Figs.~\ref{fig:SimpleAssemblyModel}-(B)(C) describe this situation.
As can be seen in training data (red points), the total manufacturing time $R(x)$ becomes significantly larger when the number $x$ of required products is greater than $x = 110$, because of of the overload of workers and machines. 
However, such structural change of the target $R(x)$ is not modeled in the simulator $r(x,\theta)$ (model misspecification).
Thus, if calibration is done without taking the covariate shift into account,  the resulting simulator makes predictions that fit well to the data in the region  $x<110$, but do not fit well in the region $x \geq 110$, as described in Fig.~\ref{fig:SimpleAssemblyModel}-(B).

Because of the first challenge of simulator calibration, exiting methods for covariate shift adaptation, which have been developed for standard statistical and machine learning approaches, cannot be directly employed for the simulator calibration problem: see e.g., \citet{Shimodaira00,Yamazaki2007,Gretton2009,Sugiyama2012} and references therein. 
On the other hand, existing approaches to likelihood-free inference, such as Approximate Bayesian Computation (ABC) methods (e.g.\cite{Csillery2010,Marin2012,Nakagome2013}), are applicable to simulator calibration, but they do not address the problem of covariate shift. 
Our approach combines these two approaches and thus enjoys the best of both worlds, offering a solution to the calibration problem with covariate shift adaptation.

This work proposes a novel approach to simulator calibration, dealing explicitly with the setting of covariate shift. 
Our approach is Bayesian, deriving a certain posterior distribution over the parameter space given observed data.
The proposed method is based on Kernel ABC \citep{Nakagome2013,FukSonGre13}, which is an approach to ABC based on  kernel mean embedding of distributions~\citep{MuaFukSriSch17}, and a certain importance-weighted kernel that works for covariate shift adaptation. 
We provide a theoretical analysis of this approach, showing that it produces a distribution over the parameter space that approximates the posterior distribution in which the ``observed data'' is predictions from the model that minimises the importance-weighted empirical risk. 
In other words, the proposed method approximates the posterior distribution whose support consists of parameters such that the resulting simulator produces a small generalization error for the test input distribution.
For instance, Fig.~\ref{fig:SimpleAssemblyModel}-(C) shows predictions obtained with our method, which fit well in the test region $x \geq 110$ as a result of covariate shift adaptation. 

This paper is organized as follows.
In Sec.~\ref{sec:Background}, we briefly review the setting of covariate shift and the framework of kernel mean embedding.
In Sec.~\ref{sec:ProposedMethod}, we present our method for simulator calibration with covariate shift adaptation, and in Sec.~\ref{sec:theory} we investigate its theoretical properties.
In Sec.~\ref{sec:experiments} we report results of numerical experiments that include calibration of the production simulator in Fig.~\ref{fig:SimpleAssemblyModel}, confirming the effectiveness of the proposed method.
Additional experimental results and all the theoretical proofs are presented in Appendix.

\section{Background}
\label{sec:Background}

We here introduce some notation and definitions used in the paper, by  reviewing the problem setting of covariate shift, and the framework of kernel mean embeddings.

\subsection{Calibration under Covariate Shift}
\label{subsec:BayesianInference}
Let $\mathcal{X} \subset \mathbb{R}^{d_\X}$ with $d_\X \in \mathbb{N}$ be a measurable subset that serves as the input space for a target system and a simulator.
Denote by $R: \X \to \mathbb{R}$ the regression function of the (unknown) target system, which is deterministic, and define the true data-generating process as
\begin{equation} \label{eq:data-gen-process}
y(x) := R(x) + e(x),
\end{equation}
where $e:\X \to \mathbb{R}$ is a (zero-mean) stochastic process that represent error in observations. 
Observed data $D_n := \{ (X_i, Y_i) \}_{i=1}^n \subset \X \times \mathbb{R}$ are assumed to be generated from the process \eqref{eq:data-gen-process} as 
\begin{eqnarray*}
 X_1,\dots,X_n \sim q_0 \ \ (\mathrm{i.i.d.}), 
\quad Y_i = y(X_i), \ \ (i = 1,\dots,n), 
\end{eqnarray*}
where $q_0$ is a probability density function on $\X$. 
We use the following notation to write the output values:
$$
Y^n := (Y_1,\dots,Y_n) \in \mathbb{R}^n.
$$

Let $\Theta \subset \R^{d_\Theta}$ with $d_\Theta \in \mathbb{N}$ be a measurable subset that serves as a parameter space.
Let $$r: \X \times \Theta \to \R$$ be a (measurable) deterministic simulation model that outputs a real value $r(x,\theta) \in \R$ given an input $x \in \X$ and a parameter $\theta \in \Theta$.
Assume that we have a prior distribution $\pi(\theta)$ on the parameter space $\Theta$.

In the setting of {\em covariate shift}, the input distribution $q_1(x)$ in the test or prediction phase is different from that $q_0(x)$ for training data $X_1,\dots,X_n$, while the input-output relationship \eqref{eq:data-gen-process} remains the same.
Thus, the expected loss (or the generalization error) to be minimized may be defined as 
\begin{eqnarray*}
L(\theta) &:=& \int \left( y(x) - r(x,\theta) \right)^2 q_1(x) dx \nonumber \\  
&=& \int \left( y(x) - r(x,\theta) \right)^2 \beta(x) q_0(x) dx \label{eq:cov-shift-risk},
\end{eqnarray*}
where $\beta: \X \to \R$ is the {\em importance weight} function, defined as the ratio of the two input densities:  
$$
 \beta (x) := q_1(x) / q_0 (x).
$$
In this work, we assume for simplicity that importance weights $\beta(X_i)$ at training inputs $X_1,\dots,X_n$ are known, or estimated in advance.
The knowledge of the importance weights is available when $q_0(x)$ and $q_1(x)$ are designed by an experimenter. 
For estimation of the importance, we refer to \citet{Gretton2009,book:Sugiyama+etal:2012} and references therein.\footnote{Note that kernel mean matching \citep{Gretton2009} is a method for estimating the importance weights $\beta(X_1),\dots,\beta(X_n)$, while it is based on kernel mean embeddings as in our method. In this sense, that approach deals with a problem different from ours.}
Using the importance weights, the expected loss can be estimated as 
\begin{equation} \label{eq:importance-weighted-emp-loss}
L_n(\theta) := \frac{1}{n} \sum_{i=1}^n \beta(X_i) \left( Y_i - r(X_i,\theta) \right)^2.
\end{equation}
Covariate shift has a strong inference of the generalization performance of an estimated model, when the true regression function $R(x)$ does not belong to the class of functions realizable by the simulation model $\{ r(\cdot, \theta) \mid \theta \in \Theta \}$, i.e., when  {\em model misspecification} occurs  \citep{Shimodaira00,Yamazaki2007}.
Such a misspecification happens in practice, since the simulation model only has a finite degree of freedom, as the parameter space is finite dimensional.
To obtain a model with a good prediction performance, one needs to use an importance-weighted loss like \eqref{eq:importance-weighted-emp-loss} for parameter estimation.


\subsection{Kernel Mean Embedding of Distributions}
This is a framework for representing  probability measures as elements in an Reproducing Kernel Hilbert Space (RKHS).
We refer to \citet{MuaFukSriSch17} and references therein for details.

Let $\Omega$ be a measurable space, $k: \Omega \times \Omega \to \mathbb{R}$ be a measurable positive definite kernel and $\mathcal{H}$ be its RKHS. 
In this framework, any probability measure $P$ on $\Omega$ is represented as a Bochner integral
\begin{equation*}  
\mu_P := \int k(\cdot, \theta)dP(\theta) \in \mathcal{H},
\end{equation*}
which is called the {\em kernel mean} of $P$.
Estimation of $P$ can be carried out by that of $\mu_P$, which is usually computationally and statistically easier, thanks to nice properties of the RKHS.
Such a strategy is justified if the mapping $P \to \mu_P$ is injective, in which case $\mu_P$ maintains all information of $P$.  
Kernels satisfying this property are called characteristic, and examples of characteristic kernels on $\Omega = \mathbb{R}^d$ include Gaussian and Mat\'ern kernels \citep{SriGreFukSchetal10}.

\section{Proposed Calibration Method}
\label{sec:ProposedMethod}

We present our approach to simulator calibration with covariate shift adaptation. 
We take a Bayesian approach,  and our target posterior distribution is described in Sec.~\ref{sec:target-posterior}.
The proposed approach consists of Kernel ABC using a certain importance-weighted kernel (Sec.~\ref{sec:KABC-weighted}) and posterior sampling with the kernel herding algorithm (Sec.~\ref{sec:sampling-KH}).

\subsection{Target Posterior Distribution}

\label{sec:target-posterior}

We define a vector-valued function $r^n: \Theta \to \mathbb{R}^n$ from the simulator $r(x,\theta)$ as
\begin{equation} \label{eq:sim-vector-valued}
r^n(\theta) := (r(X_1),\dots,r(X_n))^\top \in \mathbb{R}^n, \quad \theta \in \Theta.
\end{equation}
Let $\mathrm{supp}(\pi)$ be the support of $\pi$.
Define  $\Theta^* \subset \mathrm{supp}(\pi)$ as the set of parameters that minimize the weighted square error, i.e., for all $\theta \in \Theta^*$ we have
\begin{eqnarray}  
&& \sum_{i=1}^n \beta(X_i) (Y_i - r(X_i,\theta^*) )^2 = \nonumber \\
&& \min_{\theta \in \mathrm{supp}(\pi)} \sum_{i=1}^n \beta(X_i) (Y_i - r(X_i,\theta) )^2 . \label{eq:optimal-parameters}
\end{eqnarray}
We allow for $\Theta^*$ to contain multiple elements, but assume that they all give the same simulation outputs, which we denote by $r^* \in \mathbb{R}^n$:
\begin{equation} \label{eq:unique-sim-outputs}
r^* := r^n(\theta^*) = r^n (\tilde{\theta^*}), \quad \forall  \theta^*, \tilde{\theta^*} \in \Theta^*.
\end{equation}

Let $\vartheta \sim \pi$ be a random variable following $\pi$.
Then $r^n(\vartheta)$ is also a random variable taking values in $\mathbb{R}^n$ and its distribution is the {\em push-forward measure} of $\pi$ under the mapping  $r^n$, denoted by $r^n\pi$. 
We write the distribution of the joint random variable $$(\vartheta, r^n(\vartheta)) \in \Theta \times \R^n$$ as $P_{\Theta \R^n}$, and their marginal distributions on $\Theta$ and $\R^n$ as $P_\Theta$ and $P_{\R^n}$, respectively. 
Then by definition we have $P_\Theta = \pi$ and $P_\Rn = r^n\pi$.
 Let $$\mathrm{supp}(P_\Rn) = \mathrm{supp}(r^n \pi) = \{ r^n(\theta) \mid \theta \in \mathrm{supp}(\pi) \}$$ be the support of the push-forward measure, which is the range of the simulation outputs when the parameter is in the support of the prior.

We consider the conditional distribution on $\Theta$ induced from the joint distribution $P_{\Theta \Rn}$ by conditioning on $\y \in  \supp(P_\Rn)$, which we write 
\begin{equation} \label{eq:cond-dist}
P_\pi(\theta | \y), \quad \y \in \supp(P_\Rn)
\end{equation}
Note that, since the conditional distribution on $\R^n$ given $\theta \in \Theta$ is  the Dirac distribution at $r^n(\theta)$, one cannot use Bayes' rule to define the conditional distribution.
However, the conditional distribution \eqref{eq:cond-dist} is well-defined as a {\em disintegration}, and is uniquely determined up to an almost sure equivalence with respect to $P_\Rn$ \citep[Thm.~1 and Example 9]{ChaPol97};  see also \citet[Sec.~2.5]{CocOatSulGir17}.

It will turn out in  Sec.~\ref{sec:theory} that our approach provides an estimator for the kernel mean of the conditional distribution \eqref{eq:cond-dist} with $\y =r^*$:
\begin{equation} \label{eq:posterior}
P_\pi(\theta | r^*) 
\end{equation}
where $r^*$ is the outputs of the optimal simulator \eqref{eq:unique-sim-outputs}.
In other words, \eqref{eq:posterior} is the posterior distribution on the parameters, given that the optimal outputs $r^*$ are observed. 
Sampling from \eqref{eq:posterior} thus amounts to sampling parameters that provide the optimal simulation outputs.

Finally, we define a predictive distribution of outputs $y$ for any input point $x \in \X$ as the push-forward measure of the posterior \eqref{eq:posterior} under the mapping $r(x,\cdot): \theta \to r(x,\theta)$, which we denote by
\begin{equation} \label{eq:predictive-dist}
P_\pi(y | x, r^*).
\end{equation}

\subsection{Kernel ABC with a Weighted Kernel}
\label{sec:KABC-weighted}

Let  $k_\Theta : \Theta \times \Theta \to \R$ be a kernel on the parameter space and $\H_\Theta$ be its its RKHS. 
We define the kernel mean of the posterior \eqref{eq:posterior} as
\begin{equation} \label{eq:post-kmean-def}
\mu_{\Theta | r^*} := \int k_\Theta (\cdot,\theta) dP_\pi(\theta | r^*) \in \H_\Theta,
\end{equation}

We propose to use the following weighted kernel on $\Rn$ defined from importance weights.
As mentioned, we assume that the importance weight function $\beta (x) = q_1(x)/q_0(x)$ is known or estimated in advance. 
For $Y^n, \tilde{Y}^n \in \R^n$, the kernel is defined as 
\begin{equation}  
 k_\Rn(Y^n, \tilde{Y}^n) =  
 \exp \left( -\frac{1}{2\sigma ^2} \sum_{i=1} ^n \beta(X_i) ( Y_i - \tilde{Y}_i )^2 \right),  \label{eq:weighted-kernel}
\end{equation}
where $\sigma^2 > 0$ is a constant and a parameter of the kernel.

We apply Kernel ABC \citep{Nakagome2013} with the importance-weighted kernel defined above, to estimate the posterior kernel mean \eqref{eq:post-kmean-def}. 
First, we independently generate $m \in \mathbb{N}$ parameters from the prior $\pi(\theta)$
$$
\bar{\theta}_1,\dots,\bar{\theta}_m \sim \pi.
$$
Then for each parameter $\bar{\theta}_j,  j=1,\dots,m$, we run the simulator to generate pseudo observations at $X_1,\dots,X_n$:
$$
\bar{Y}^n_j := r^n(\bar{\theta}_j), \quad j = 1,\dots,m,
$$
where $r^n: \Theta \to \Rn$ is defined in \eqref{eq:sim-vector-valued}.
Then an estimator of the kernel mean \eqref{eq:post-kmean-def} is given by
\begin{eqnarray}
&&\hat{\mu}_{\Theta | r^*} := \sum_{j=1} ^m w_j k_\Theta(\cdot, \bar{\theta}_j) \ \in \H_\Theta,  \label{eq:kernel_postmean}  \label{eq:w2} \\
&& (w_1,..., w_m)^\top  \nonumber 
:= (G + m \varepsilon I_m)^{-1} {\bf{k}}_\Rn (Y^n) \in \mathbb{R}^m,
\end{eqnarray}
where $I_m \in \mathbb{R}^{m \times m}$ is the identity and $\varepsilon > 0$ is a regularization constant; the vector ${\bf k}_\Rn(Y^n) \in \mathbb{R}^m$ and the Gram matrix $G \in \mathbb{R}^{m \times m}$ are computed from the kernel $k_\Rn$ in \eqref{eq:weighted-kernel} with the observed data $Y^n$ as 
\begin{eqnarray*}
{\bf k}_\Rn(Y^n) &:=& (k_\Rn(\bar{Y}^n_1, Y^n) ,\dots, k_\Rn(\bar{Y}^n_m, Y^n))^\top  \in \mathbb{R}^m \\ \label{eq:kernel1}
G &:=&  (k_\Rn(\bar{Y}_j ^n, \bar{Y}_{j'} ^n)) _{j, j' = 1} ^m \in \mathbb{R}^{m\times m}. \label{grammian}
\end{eqnarray*}

\subsection{Posterior Sampling with Kernel Herding}
\label{sec:sampling-KH}
We apply Kernel herding~\citep{CheWelSmo10}, a deterministic sampling method based on kernel mean embedding, to generate parameters $\check{\theta}_1, ..., \check{\theta}_m \in \Theta$ from the posterior kernel mean $\hat{\mu}_{\Theta | r^*}$ in  (\ref{eq:kernel_postmean}).
The procedure is as follows.
The initial point $\check{\theta}_1$ is generated as $\check{\theta}_1 := \argmax_{\theta \in \Theta} \hat{\mu}_{\Theta | r^*}(\theta)$.
Then the subsequent points $\check{\theta}_t$, $t = 2,\dots,m$, are generated sequentially as
$$
\check{\theta}_t := \argmax_{\theta \in \Theta} \hat{\mu}_{\Theta | r^*}(\theta) - \frac{1}{t} \sum_{j=1}^{t-1} k_\Theta(\theta, \check{\theta}_j).
$$
These points are a sample from the approximate posterior, in the sense that they satisfy  $\|  \hat{\mu}_{\Theta | r^*} - \frac{1}{t} \sum_{j=1}^t k_\Theta(\cdot, \check{\theta}_j)   \|_{\H_\Theta} = O(t^{-1/2})$ under a mild condition \citep{BacJulObo12}.

{\bf Prediction.} Let $x \in \X$ be any test input location, and recall that the predictive distribution $P_\pi(y|x, r^*) $ in \eqref{eq:predictive-dist} is defined as the push-forward measure of the posterior $P_\pi(\theta | r^*)$ under the mapping $r(x, \cdot)$.
Therefore, predictive outputs can be obtained simply by running simulations with the posterior samples $\check{\theta}_1, \dots, \check{\theta}_m$:
$$
r(x,\check{\theta}_1), \dots, r(x, \check{\theta}_m),
$$
and the predictive distribution is approximated by the empirical distribution
$$
\hat{P}_\pi(y|x, r^*)  := \frac{1}{m} \sum_{j=1}^m \delta(y - r(x, \check{\theta}_j)),
$$
where $\delta(\cdot)$ is the Dirac distribution at $0$.

\section{Theoretical Analysis}
\label{sec:theory}

To analyze the proposed method, we first express the estimator \eqref{eq:kernel_postmean} in terms of {\em covariance operators} on the RKHSs, which is how the estimator was originally proposed \citep{Song2009,Nakagome2013}.
To this end, define joint random variables $(\vartheta, \y) \in \Theta \times \Rn$ by 
$$
\vartheta \sim \pi, \quad \y:= r^n(\vartheta),
$$
where $r^n :\Theta \to \R^n$ is defined in \eqref{eq:sim-vector-valued}.
Let $\H_\Theta$ and $\H_\Rn$ be the RKHSs of $k_\Theta$ and $k_\Rn$, respectively. 

Covariance operators $C_{\vartheta \y} : \H_\Rn \to \H_\Theta$ and $C_{\y \y} : \H_\Rn \to \H_\Rn$  are then defined as
\begin{eqnarray*}
 C_{\vartheta \y} f &:=& \E[ k_\Theta(\cdot,\vartheta) f(\y) ]  \in \H_\Theta, \quad f \in \H_\Rn, \\
C_{\y \y}f &:=& \E[ k_\Rn(\cdot,\y) f(\y) ]  \in \H_\Rn, \quad f \in \H_\Rn.
\end{eqnarray*}
Note that parameter-data pairs $(\bar{\theta}_j, \bar{Y}^n_j)_{j=1}^m = (\bar{\theta}_j, r^n(\bar{\theta}_j) )_{j=1}^m \subset \Theta \times \R^n$  in Kernel ABC (Sec.~\ref{sec:KABC-weighted}) are i.i.d.~copies of the random variables $(\vartheta, \y)$.
Thus empirical covariance operators $\hat{C}_{\vartheta \y} : \H_\Rn \to \H_\Theta$ and $\hat{C}_{\y \y} : \H_\Rn \to \H_\Rn$ are defined as
\begin{eqnarray*}
&& \hat{C}_{\vartheta \y} f := \frac{1}{m} \sum_{j = 1}^m k_\Theta(\cdot,\bar{\theta}_j) f(\bar{Y}^n_j), \quad f \in \H_\Rn, \\
&& \hat{C}_{\y \y}f := \frac{1}{m} \sum_{j = 1}^m k_\Rn(\cdot, \bar{Y}^n_j ) f(\bar{Y}^n_j ), \quad f \in \H_\Rn.
\end{eqnarray*}
The estimator \eqref{eq:kernel_postmean} is then expressed as
\begin{equation} \label{eq:est-cov-expression}
\hat{\mu}_{\Theta | r^*} = \hat{C}_{\vartheta \y}( \hat{C}_{\y \y} + \varepsilon I )^{-1} k_\Rn(\cdot,Y^n).
\end{equation}
See the above original references as well as \citet{SonFukGre13,FukSonGre13,MuaFukSriSch17} for the derivation.

Recall that $Y^n$ is the observed data from the real process. 
The issue is that, in our setting, $Y^n$ may {\em not} lie in the support of the distribution $P_\Rn$ of $\y = r^n(\vartheta)$, since the simulation model $r(\theta,x)$ is misspecified, i.e., there exists no $\theta \in \Theta$ such that $R(x) = r(x,\theta)$ for all $x \in \X$.
The misspecified setting where $Y^n \not\in \supp(P_\Rn)$ has not been studied in the literature on kernel mean embeddings, and therefore existing theoretical results on conditional mean embeddings 
\citep{GruLevBalPatetal12,Fuk15,SinSahGre19} are not directly applicable.
Our theoretical contribution is to study the estimator \eqref{eq:est-cov-expression} in this misspecified setting, which may be of general interest.

\subsection{Projection and Best Approximation}
Let $\H_{\y} \subset \H_\Rn$ be the Hilbert subspace of $\H_\Rn$ defined as the completion of the linear span of functions $k_\Rn(\cdot,\tilde{Y}^n)$ with $\tilde{Y}^n$ from the support of $P_\Rn$: 
\begin{equation} \label{eq:Hilbert-subspace-support}
\H_{\y} := \overline{\mathrm{span}\left\{ k_\Rn(\cdot,\tilde{Y}^n) \mid  \tilde{Y}^n\in \mathrm{supp}(P_\Rn) \right\}},
\end{equation}
where the closure is taken with respect to the norm of $\H_\Rn$.
In other words, every $h \in \H_\y$ may be written in the form $h = \sum_{\ell = 1}^\infty \alpha_\ell k_\Rn(\cdot, \tilde{Y}_\ell^n)$ for some $(\alpha_\ell)_{\ell = 1}^\infty \subset \R$ and $(\tilde{Y}^n_\ell)_{\ell = 1}^\infty \subset \supp(P_\Rn)$ such that $\|h\|_{\H_\Rn}^2 = \sum_{\ell, j = 1}^\infty \alpha_\ell \alpha_j k_\Rn(\tilde{Y}^n_\ell, \tilde{Y}^n_j) < \infty$.

Since $\H_y$ is a Hilbert subspace, one can consider the orthogonal projection of $k_\Rn(\cdot,Y^n)$, the ``feature vector''  of the observed data $Y^n$, onto $\H_{\y}$, which is uniquely determined and denoted by
\begin{equation} \label{eq:projection}
    h^* := \argmin_{h \in \H_{\y}} \| h - k_\Rn(\cdot,Y^n) \|_{\H_\Rn}.
\end{equation}
Then $k_\Rn(\cdot,Y^n)$ can be written as $$
k_\Rn(\cdot,Y^n) = h^* + h_{\perp},
$$
where $h_{\perp} \in \H_\Rn$ is orthogonal to $\H_{\y}$.

Note that the estimator \eqref{eq:est-cov-expression} is an approximation to the following population expression:
\begin{equation} \label{eq:pop-CME-reg}
C_{\vartheta \y}( C_{\y \y} + \varepsilon I )^{-1} k_\Rn(\cdot,Y^n).    
\end{equation}
Our first result below shows that \eqref{eq:pop-CME-reg} can be written in terms of the projection \eqref{eq:projection}.
\begin{lemma} \label{lemma:CME-projection}
Let $k_\Theta$ be a bounded and continuous kernel and assume that $0 < \beta(X_i) < \infty$ holds for all $i = 1,\dots,n$.
Then \eqref{eq:pop-CME-reg} is equal to 
$$
C_{\vartheta \y}( C_{\y \y} + \varepsilon I )^{-1} h^*
$$
\end{lemma}

We make the following identifiability assumption.
It is an assumption on the observed data $Y^n$ (or the data generating process \eqref{eq:data-gen-process}), the simulation model $r(x,\theta)$ and the kernel  $k_\Rn$  (or the importance weight function $\beta(x) = q_1(x)/q_0(x)$; see the definition of $k_\Rn$ in \eqref{eq:weighted-kernel}).
\begin{assumption} \label{as:identifiability}
There exists some $\tilde{Y}^n \in \supp(P_\Rn)$ such that $k_\Rn(\cdot,\tilde{Y}^n) = h^*$, where $h^*$ is the orthogonal projection of $k_\Rn(\cdot,Y^n)$ onto the subspace $\H_{\y}$ in \eqref{eq:projection}. 
\end{assumption}
The assumption states that the orthogonal projection of the feature vector $k_\Rn(\cdot,Y^n)$ of observed data $Y^n$ onto $\H_\y$ lies in the set 
\begin{eqnarray*}
\lefteqn{\{ k_\Rn(\cdot,  \tilde{Y}^n) \mid  \tilde{Y}^n \in \supp(P_\Rn)\}} \\
&=& \{ k_\Rn(\cdot,  r^n(\theta)) \mid  \theta \in \supp(\pi)\}.
\end{eqnarray*}
Thus the assumption implies that the best approximation $h^*$ of the observed data is given by the simulation model with some parameter $\theta^* \in \mathrm{supp}(\pi)$, i.e., $h^* = k_\Rn(\cdot, r^n(\theta^*))$. 
Such $\theta^*$ satisfies
\begin{eqnarray*}
\theta^*  &\in& \argmin_{\theta \in \mathrm{supp}(\pi)} \left\| k_\Rn(\cdot,Y^n) -  k_\Rn(\cdot, r(\cdot,\theta)) \right\|_{\H_\Rn}^2 \\
&=& \argmax_{\theta \in \mathrm{supp}(\pi)} k_\Rn (Y^n, r(\cdot,\theta)) \\
&=& \argmax_{\theta \in \mathrm{supp}(\pi)} \exp \left( -\frac{1}{2\sigma ^2} \sum_{i=1}^n \beta(X_i) (Y_i - r(X_i,\theta) )^2 \right) \nonumber \\
&=& \argmin_{\theta \in \mathrm{supp}(\pi)} \sum_{i=1}^n \beta(X_i) (Y_i - r(X_i,\theta) )^2, 
\end{eqnarray*}
where the last identity follows from the exponential function being monotonically increasing.
This shows that, under Assumption \ref{as:identifiability}, the parameter $\theta^*$ realizing the projection is a least weighted-squares solution, and thus belongs to the set $\Theta^*$ defined in \eqref{eq:optimal-parameters}. 
Moreover, since $h^*$ is uniquely determined, so is the simulation outputs $r^* := r^n(\theta^*)$, in the sense of \eqref{eq:unique-sim-outputs}. 

By these arguments, Lemma \ref{lemma:CME-projection} and Assumption \ref{as:identifiability} lead to the following result.
\begin{theorem} \label{theo:cme-population-misspecified}
Suppose that the assumptions in Lemma \ref{lemma:CME-projection} and Assumption \ref{as:identifiability} hold.
Let $r^* := r^n(\theta^*)$ where $\theta^*$ is any element satisfying \eqref{eq:optimal-parameters}.  
Then \eqref{eq:pop-CME-reg} is equal to
$$
C_{\vartheta \y}( C_{\y \y} + \varepsilon I )^{-1} k_\Rn(\cdot,r^*).
$$
\end{theorem}
Theorem \ref{theo:cme-population-misspecified} suggests that the estimator \eqref{eq:est-cov-expression} would behave as if the observed data is the optimal simulation outputs $r^*$ obtained as a best approximation for the given data $Y^n$.
The convergence result presented below shows that this is indeed the case.

To state the result, we define a function $G : \supp(P_\Rn)\times \supp(P_\Rn)\to \mathbb{R}$ as
\begin{eqnarray} \label{eq:G-func}
G(Y_a^n, Y_b^n) &:=& \E [k_\Theta(\vartheta, \vartheta ) | \y = Y_a^n, \y' =Y_b^n  ], \\
&=&  \E [k_\Theta(\vartheta, \vartheta ) | r^n(\vartheta) = Y_a^n, r^n(\vartheta') =Y_b^n  ], \nonumber
\end{eqnarray}
where $(\vartheta', \y')$ is an independent copy of $(\vartheta, \y)$.

The following result shows that \eqref{eq:est-cov-expression} (or \eqref{eq:kernel_postmean}) is a consistent estimator of
the kernel mean $ \mu_{\Theta | r^* }$ \eqref{eq:post-kmean-def} of the posterior $P_{\pi}(\theta|r^*)$.
It is obtained by extending the result of \citet[Theorem 1.3.2]{Fuk15} to the misspecified setting where $Y^n \not\in \supp(P_\Rn)$ by using Theorem\ref{theo:cme-population-misspecified}.
The assumptions made are essentially the same those in \citet[Theorem 1.3.2]{Fuk15}.
Below $\mathrm{Range}(C_{\y \y} \otimes C_{\y \y})$ denotes the range of the tensor-product operator $C_{\y \y} \otimes C_{\y \y}$ on the tensor-product RKHS $\H_\Rn \otimes \H_\Rn $ (see Appendix for details).

\begin{theorem}  \label{theo:convergence}
Suppose that the assumptions in Lemma \ref{lemma:CME-projection} and Assumption \ref{as:identifiability} hold.
Assume that the eigenvalues $\lambda_1 \geq \lambda_2 \geq \cdots \geq 0$ of $C_{\y \y}$ satisfy $\lambda_i \leq \beta i^{-b}$ for all $i \in \mathbb{N}$  for some constants $\beta > 0$ and $b > 1$, and that the function $G$ in \eqref{eq:G-func} satisfies $G \in \mathrm{Range}(C_{\y \y} \otimes C_{\y \y})$.
Let $C > 0$ be any fixed constant, and set the regularization constant  $\varepsilon := \varepsilon_m := C m^{- \frac{b}{1+4b}}$  of $\hat{\mu}_{\Theta | r^* }$ in \eqref{eq:est-cov-expression} (or \eqref{eq:kernel_postmean}).
Then we have 
$$
\left\| \hat{\mu}_{\Theta | r^* } - \mu_{\Theta | r^* } \right\|_{\H_\Theta}  = O_p\left(  m^{- \frac{b}{1+4b}} \right) \ ( m \to \infty).
$$
\end{theorem}

\section{Experiments} \label{sec:experiments}

We first explain the setting common for all the experiments. 
In each experiment, we consider both regression problems with and without covariate shift, to see whether the proposed method can deal with covariate shift. 
In the latter case, which we call ``ordinary regression,'' we set the importance weights to be constant, $\beta(X_i) = 1$ ($i=1,...,n$).
The noise process $e(x)$ in \eqref{eq:data-gen-process} is independent Gaussian $\varepsilon 
\sim N(0,\sigma_{\rm noise}^2)$.
We write $N(a,b)$ for the normal distribution with mean $a$ and variance $b$; the multivariate version is denoted similarly.

For the proposed method, we used a Gaussian kernel $k_\Theta(\theta,\theta') = \exp(- \| \theta - \theta' \|^2 / 2 \sigma^2_\Theta)$ for the parameter space, where $\sigma^2_\Theta > 0$ is a constant.
We set the constants $\sigma^2, \sigma_\Theta^2  > 0$ in the kernels $k_\Rn$ and $k_\Theta$ by the median heuristic \cite[e.g.][]{median_heu_2018} using the simulated pairs $(\bar{\theta}_j, \bar{Y}^n_j)_{j=1}^m$.

For comparison, we used Markov Chain Monte Carlo (MCMC) for posterior sampling, more specifically the Metropolis-Hastings (MH) algorithm. 
For this competitor, we assume that the noise process $e(x)$ in \eqref{eq:data-gen-process} is known, so that the likelihood function is available in MCMC (which is of the form $\exp( - \sum_{i=1}^n \beta(X_i) \left( Y_i - r(X_i,\theta) \right)^2 / 2\sigma_{\rm noise}^2)$ up to constant).  
In this sense, we give an unfair advantage for MH over the proposed method, as the latter does not assume the knowledge of the noise process, which is usually not available in practice.

For evaluation, we compute Root Mean Square Error (RMSE) in prediction for each method (and for a different number of simulations, $m$) as follows. 
Test input locations $\tilde{X}_1, \dots, \tilde{X}_n$ are generated from $q_0(x)$ in the case of ordinary regression, and from $q_1(x)$ in the covaraite shift setting. 
After sampling parameters $\check{\theta}_1, \dots, \check{\theta}_m$ with the method for evaluation, the RMSE is computed as   $( \frac{1}{n} \sum_{i=1}^n ( R(\tilde{X}_i) - \frac{1}{m} \sum_{j=1}^m r(\tilde{X}_i, \check{\theta}_j ) ) ^2 )^{1/2}$.

\subsection{Synthetic Experiments}
\label{subsec:SyntheticExperiment}


\label{subsec:SettingOfExperiment}

We consider the problem setting of the benchmark experiment in ~\cite{Shimodaira00}.

{\bf Setting.} 
The input space is $\X = \R$, and the data generating process \eqref{eq:data-gen-process} is given by $R(x) = -x+x^3$ and $e(x) = \epsilon$ with $\epsilon \sim N(0,2)$ being an independent noise.
The simulation model is defined by $r(x, \theta) = \theta_0 + \theta_1 x$, where $\theta = (\theta_1,\theta_2)^\top \in \Theta =\R^d$. 
For demonstration, we treat this model as intractable, i.e., we assume that only evaluation of  function values $r(x,\theta)$ is possible once $x$ and $\theta$ are given.
The input densities $q_0(x)$ and $q_1(x)$ for for training and prediction are those of $N(0.5,0.5)$ and $N(0, 0.3)$, respectively.
We define the prior as multivariate Gaussian $\pi = N({\bf 0},5 I_2)$, where $I_2 \in \R^{2\times 2}$ is the identity.  
We set the size of training data $(X_i,Y_i)_{i=1}^n$ as $n = 100$.



\begin{figure}[t]
  \centering
  \includegraphics[width=\linewidth]{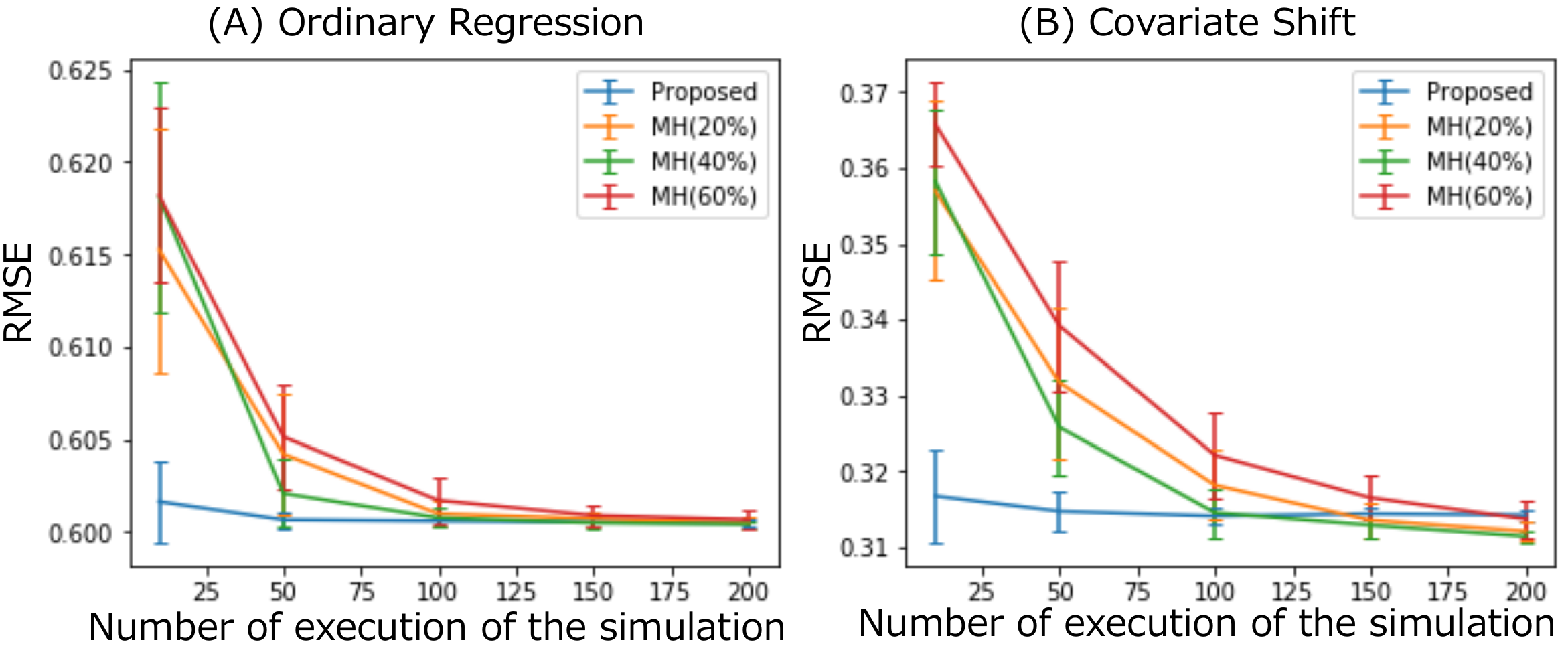}
  \caption{
  RMSEs for (A) ordinary and (B) covariate shift cases, as a function of the number $m$ of simulations, given by the proposed method (blue) and the MH algorithm with different acceptance ratios about 20\% (orange), 40\% (red), and 60\% (green).}
  \label{fig:nsim_vs_rmse}
\end{figure}

{\bf Results.}
Figure~\ref{fig:nsim_vs_rmse} shows RMSEs for (A) ordinary regression and (B) covariate shift as a function of the number $m$ of simulations, with the means and standard deviations calculated from 30 independent trials.
For the proposed method, we set the regularization constant to be $\varepsilon=1.0$.
We set the proposal distribution of MH to be $N({\bf 0}, \sigma_{p}^2 I_2)$ with $\sigma_{p}$ being 0.08,  0.06, and 0.03, which were tuned so that the acceptance ratios become about 20\%, 40\%, and 60\% respectively.
In the horizontal axis, the number of simulations for MH is the number of all MCMC steps (which all require running the simulator) including burn-in and rejected executions. 
For MH, we used the first 10\% MCMC steps for burn-in, and excluded them for predictions.
The results show that the proposed method is more efficient than MH, in the sense that it gives better predictions than MH based on a small number of simulations.
This is a promising property, since real-world simulators are often computationally expensive, as is the case for the experiment in the next section.


%

\subsection{Experiments on Production Simulator}
\label{subsec:ExperimentProductSimulator}


We performed experiments on the manufacturing process simulator mentioned in Sec.~\ref{sec:Introduction} (Fig.~\ref{fig:SimpleAssemblyModel}), and a more sophisticated production simulator with 12 parameters.
We only describe the former here, and report the latter in the Appendix due to the space limitation. 

{\bf Setting.}
 We used a simulator constructed with {\em WITNESS}, a popular software package for production simulation (\url{https://www.lanner.com/en-us/}). 
 We refer to Sec.~\ref{sec:Introduction} for an explanation of the simulator.
 This simulator $r(x,\theta)$ has 4 parameters $\theta \in \Theta  \subset \R^4$.
 The input space for regression is $\X  = (0,\infty)$.  

The data generating process \eqref{eq:data-gen-process} is defined as $R(x) = r(x, \theta^{(0)})$ for $x < 110$ and $R(x) = r(x, \theta^{(1)})$ for $x \geq 110$, where $\theta^{(0)} := ( 2, 0.5, 5, 1 )^\top$ and $\theta^{(1)} := (3.5, 0.5, 7, 1)^\top$; the noise model is an independent noise $e(x)= \epsilon \sim N(0, 30)$.
The input densities are defined as $q_0(x) = N(100, 10)$ (training) and $q_1(x) = N(120, 10)$ (prediction). 
We constructed this model so that the two regions $x < 110$ and $x \geq 110$ correspond to those for training and prediction, respectively, with $\theta^{(0)}$ and $\theta^{(1)}$ being the ``true'' parameters in the respective regions. 
We defined the prior $\pi(\theta)$ as the uniform distribution over $\Theta := [0,5] \times [0,2] \times [0,10] \times [0,2] \subset \R^4$. 
The size of training data $(X_i, Y_i)_{i=1}^n$ (which are described in  Fig.~\ref{fig:SimpleAssemblyModel}~(B)(C) as red points) is $n=50$.

\begin{figure}[t]
  \centering
  \includegraphics[width=\linewidth]{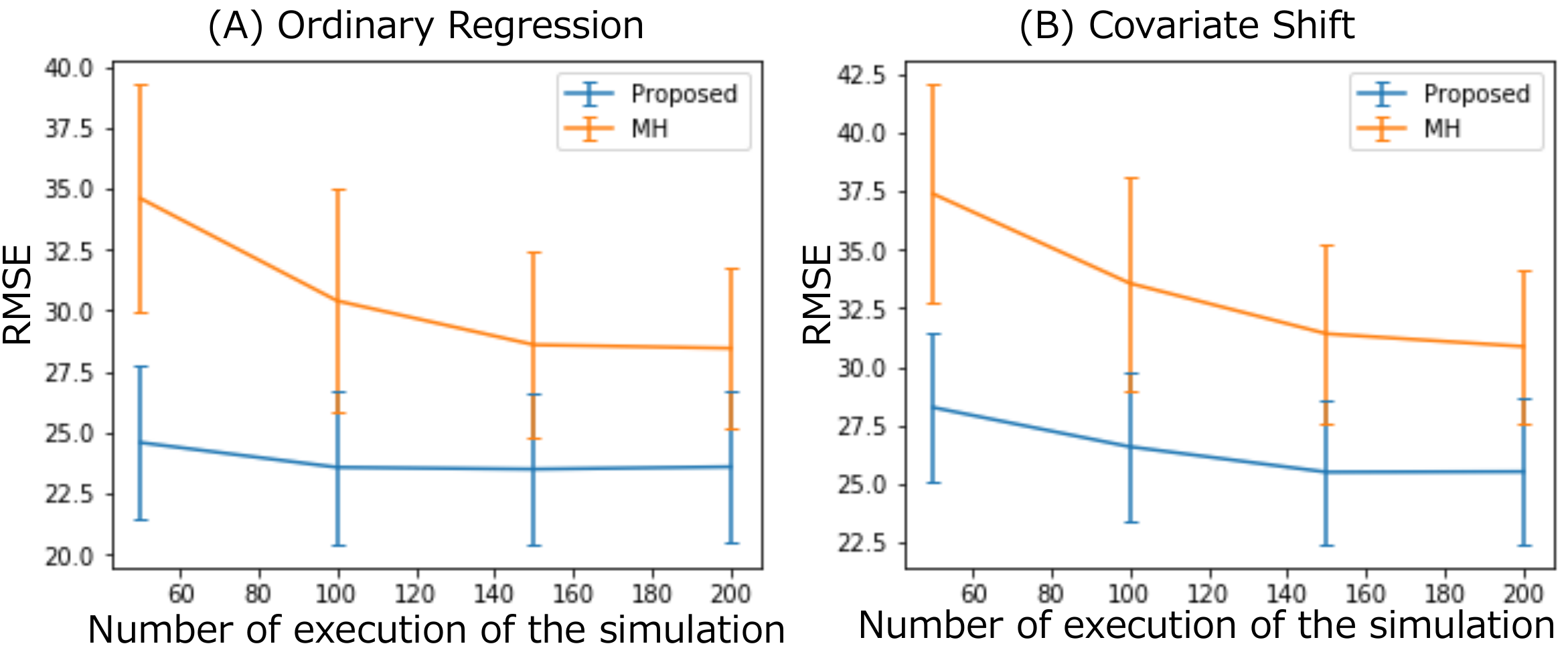}
  \caption{RMSEs  in the (A) ordinary and (B) covariate shift settings, as a function of the number $m$ of simulations, for the proposed method (blue) and MH (orange).}
  \label{fig:nsim_vs_rmse_simpleprod}
\end{figure}

{\bf Results.}
\label{subsubsec:ResultOfSimpleExperiment}
Figure~\ref{fig:nsim_vs_rmse_simpleprod}  shows the averages and standard deviations of RMSEs for the proposed method and MH of 10 independent trials, changing the number $m$ of simulations.
We set the regularization constant of the proposed method as $\varepsilon=0.01$, and the proposal distribution of MH as $N({\bf 0}, 0.03^2 I_4)$, which was tuned to make the acceptance about 40\%.\footnote{In this experiment one simulation is computationally expensive and takes about 2 seconds with the authors' PC, so we decided to only use this acceptance rate, given that the it performed the best in the previous experiment.}
The results show that the proposed method is more accurate than MH with a small number of simulations, even though the latter used the full knowledge of the data generating process \eqref{eq:data-gen-process}.

\begin{figure}[t]
  \centering
  \includegraphics[width=8.3cm]{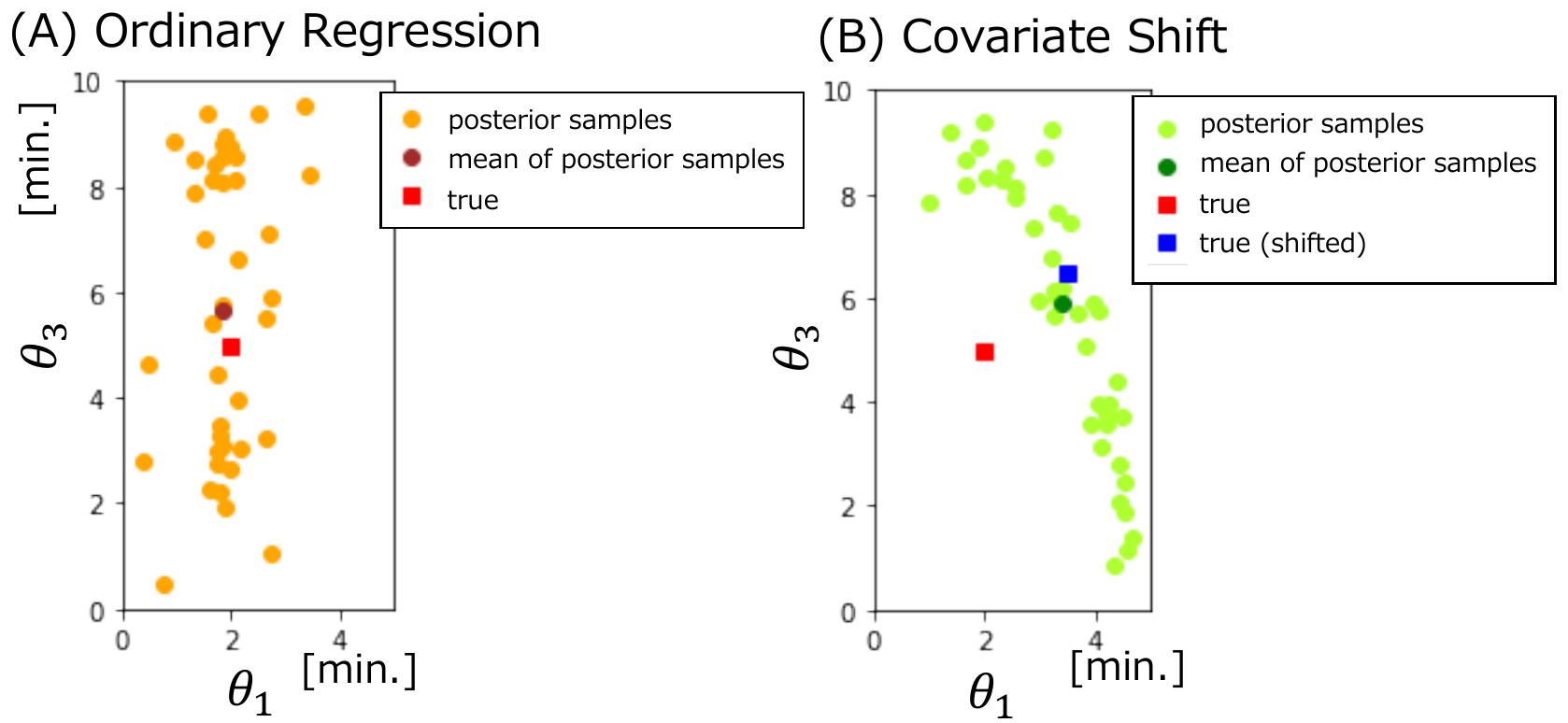}
  \caption{
  Parameters $\check{\theta}_1,\dots,\check{\theta}_m$ generated from the proposed method, in the subspace of coordinates of $\theta_1$ and $\theta_3$.
  (A): Ordinary regression: the generated parameters (orange), the mean of them (brown), and the ``true'' parameter $\theta^{(0)}$ for the training region $x  < 110$ (red).
  (B) Covariate shift: the generated parameters (light green), the mean of them (green),  and the ``true'' parameter $\theta^{(1)}$ for the prediction region $x \geq 110$ (blue, ``true shifted'').
  }
  \label{fig:result_SimpleAssemblyModel}
\end{figure}

Fig.~\ref{fig:result_SimpleAssemblyModel}~(A) and (B) describe parameters $\check{\theta}_1, \dots, \check{\theta}_m$ generated in one run of the proposed method in the ordinary and covariate shift settings, respectively; the corresponding predictive outputs are shown in Fig.~\ref{fig:SimpleAssemblyModel}~(B) and (C).
In both settings, the estimated posterior mean is located near the ``true'' parameter of each scenario.
Fig.~\ref{fig:result_SimpleAssemblyModel}~(A) and (B) also demonstrate how our method might be useful for sensitivity analysis.
Our method generates parameters $\check{\theta}_1, \dots, \check{\theta}_m$ so as to approximate the posterior $P_{\pi}(\theta|r^*)$, where $r^*$ is ``optimal'' simulation outputs. 
Therefore, the more variation in the coordinate $\theta_1$ indicates that the value of $\theta_1$ is not very important to obtain optimal simulation outputs. 
But a comparison between (A) and (B) indicates that, under covariate shift, there should be small correlation between $\theta_1$ and $\theta_3$ to obtain optimal simulation outputs.

\subsection*{Acknowledgements}
We would to thank the reviewers and the area chair for their constructive feedback.

\bibliographystyle{apalike}
\bibliography{Bibfile}


\newpage 

\onecolumn

\appendix 
\begin{center}
    {\bf \Large Supplementary Materials}  
\end{center}
\begin{center}
{\bf  \large Simulator Calibration under Covariate Shift with Kernels}    
\end{center}

\section{Proofs}

\subsection{Proof of Lemma \ref{lemma:CME-projection}}

First we note that from the assumption $0 < \beta(X_i) < \infty$ for all $i=1,\dots,n$, the importance-weighted kernel \eqref{eq:weighted-kernel} is continuous on $\Rn$.
Therefore \citet[Lemma 4.33]{SteChr2008} implies that the RKHS $\H_\Rn$ of $k_\Rn$ is separable.

To prove Lemma \ref{lemma:CME-projection}, we need the following result.
\begin{lemma} \label{lemma:null-space}
Suppose that the assumptions in Lemma \ref{lemma:CME-projection} hold.
Let $(\phi_i)_{i=1}^\infty \subset \H_\Rn$ be  the eigenfunctions of the covariance operator $C_{\y \y}$ associated with positive eigenvalues, and let  $(\tilde{\phi}_j)_{j=1}^\infty \subset \H_\Rn$ be an ONB of the null space of $C_{\y \y}$.
Then $ \tilde{\phi}_j(\tilde{Y}^n) = 0$ holds for $P_\Rn$-almost every $\tilde{Y}^n \in \Rn$.
\end{lemma}

\begin{proof}

By definition of $\tilde{\phi}_j$, its holds that
$$
0 = C_{\y \y} \tilde{\phi}_j = \int k_\Rn(\cdot, \tilde{Y}^n) \tilde{\phi}_j(\tilde{Y}^n) dP_\Rn(\tilde{Y}^n) =:   \int k_\Rn(\cdot, \tilde{Y}^n) d\nu(\tilde{Y}^n),
$$
where the measure $\nu$ is defined by $d\nu(\tilde{Y}^n) := \tilde{\phi}_j(\tilde{Y}^n) dP_\Rn(\tilde{Y}^n)$.
Since the kernel $k_\Rn$ is bounded on $\Rn$, $\H_\Rn$ consists of bounded functions, and thus $\tilde{\phi}_j \in \H_\Rn$ is bounded.
Therefore $\nu$  a finite measure.  
But since $k_\Rn$ is a Gaussian kernel (see \eqref{eq:weighted-kernel}), it is $c_0$-universal, and so \citet[Proposition 2]{SriFukLan11} and the integral being zero imply that $\nu$ is the zero measure. 
Thus for $\nu$ to be the zero measure, $ \tilde{\phi}_j(\tilde{Y}^n) = 0$  should hold for $P_\Rn$-almost every $\tilde{Y}^n$, which concludes the proof.

\end{proof}

We now prove Lemma \ref{lemma:CME-projection}.
\begin{proof}
Let $(\phi_i)_{i=1}^\infty \subset \H_\Rn$ be  the eigenfunctions of the covariance operator $C_{\y \y}$ associated with positive eigenvalues $\lambda_1 \geq \lambda_2 \geq \dots > 0$, and let  $(\tilde{\phi}_j)_{j=1}^\infty \subset \H_\Rn$ be an ONB of the null space of $C_{\y \y}$.
To prove the assertion, we first show that (a) $\left< \phi_i, h_\perp \right> = 0$ for every $\phi_i$, and that (b) $C_{\vartheta \y} \tilde{\phi}_j = 0$ for every $\tilde{\phi}_j$.   

(a) By definition of $\phi_i$, it can be written as
$$
 \phi_i = \lambda_i^{-1} C_{\y \y} \phi_i =  \lambda_i^{-1}\int k_\Rn(\cdot, \tilde{Y}^n) \phi_i(\tilde{Y}^n) d P_\Rn(\tilde{Y}^n).
$$
Therefore, 
\begin{eqnarray*}
\left< \phi_i, h_\perp\right>_{\H_\Rn} 
&=& \left< \lambda_i^{-1}\int k_\Rn(\cdot, \tilde{Y}^n) \phi_i(\tilde{Y}^n) d P_\Rn(\tilde{Y}^n), h_\perp\right>_{\H_\Rn} \\
&=& \lambda_i^{-1}\int \left<  k_\Rn(\cdot, \tilde{Y}^n) , h_\perp\right>_{\H_\Rn}  \phi_i(\tilde{Y}^n) d P_\Rn(\tilde{Y}^n) = 0,
\end{eqnarray*}
where the last identity follows from $ \left<  k_\Rn(\cdot, \tilde{Y}^n) , h_\perp\right>_{\H_\Rn} = 0$ for $\tilde{Y}^n \in \supp(P_\Rn)$, which follows from the definition of $h_\perp$. 

(b) We have
\begin{eqnarray*}
C_{\vartheta \y} \tilde{\phi}_j &=& \int k_\Theta(\cdot, \theta) \tilde{\phi}_j (\tilde{Y}^n) dP_{\Theta \Rn} (\theta, \tilde{Y}^n) \\
&=&  \int  \left( \int k_\Theta(\cdot, \theta) dP_\pi(\theta|\tilde{Y}^n)\right) \tilde{\phi}_j (\tilde{Y}^n) dP_{\Rn} (\tilde{Y}^n) = 0,
\end{eqnarray*}
 where the last identity follows from Lemma \ref{lemma:null-space}.

We now prove the assertion. 
By using (a) and (b), we obtain
\begin{eqnarray*} 
 && C_{\vartheta \y}  (C_{\y \y} + \varepsilon I)^{-1} k_\Rn(\cdot, Y^n) \\
 &=& C_{\vartheta \y}  (C_{\y \y} + \varepsilon I)^{-1} (h^* + h_\perp) \\
 &=& C_{\vartheta \y}  \sum_{i  = 1}^\infty ( \lambda_i + \varepsilon )^{-1} \left< h^*, \phi_i \right>_{\H_\Rn} \phi_i  +  C_{\vartheta \y} \sum_{j = 1}^\infty \varepsilon^{-1} \left< h^* + h_\perp, \tilde{\phi}_j \right>_{\H_\Rn} \tilde{\phi}_j \\
 &=& C_{\vartheta \y} \sum_{i =1}^\infty ( \lambda_i + \varepsilon )^{-1} \left< h^*, \phi_i \right>_{\H_\Rn} \phi_i  \\
  &=& C_{\vartheta \y} \sum_{i =1}^\infty ( \lambda_i + \varepsilon )^{-1} \left< h^*, \phi_i \right>_{\H_\Rn} \phi_i   +  C_{\vartheta \y} \sum_{j = 1}^\infty \varepsilon^{-1} \left< h^* , \tilde{\phi}_j \right>_{\H_\Rn} \tilde{\phi}_j   \\
&=&  C_{\vartheta \y}( C_{\y \y} + \varepsilon I )^{-1} h^*,
\end{eqnarray*}
which completes the proof.

\end{proof}

\subsection{Proof of Theorem \ref{theo:convergence}}
Theorem \ref{theo:convergence} can be easily proven by combining the proof idea of \citet[Theorem 1.3.2]{Fuk15} and Theorem \ref{theo:cme-population-misspecified}, but for completeness we present the proof.

Before presenting, we introduce some notation and definitions. 
Below $\| A \| $ for an operator $A$ denotes the operator norm.
$\H_\Rn \otimes \H_\Rn$ denotes the tensor-product RKHS of $\H_\Rn$ and $\H_\Rn$, which is the RKHS of the product kernel $k_{\Rn \times \Rn}: \Rn \times \Rn \to \R$  defined by $k_{\Rn \times \Rn}( (Y_a^n, \tilde{Y}_a^n), (Y_b^n, \tilde{Y}_b^n) ) = k_\Rn((Y_a^n, Y_b^n)) k_\Rn((\tilde{Y}_a^n, \tilde{Y}_b^n))$.
$C_{\y \y} \otimes C_{\y \y} : \H_\Rn \otimes \H_\Rn \to \H_\Rn \otimes \H_\Rn$ is the covariance operator defined by
$$
C_{\y \y} \otimes C_{\y \y} F := \E[ k_{  \Rn \times \Rn} ( \cdot , (\y, \y' ) ) F(\y, \y') ], \quad F \in \H_\Rn \otimes \H_\Rn,
$$
where $\y'$ is an independent copy of the random variable $\y$.

Note that the covariance operator $C_{\vartheta \y}$ satisfies $\left< C_{\vartheta \y} f, g \right>_{\H_\Theta} = \E[f(\y)g(\vartheta)]$ for any $f \in \H_\Rn$ and $g \in \H_\Theta$.
Similarly, $C_{\y \y}$ satisfies $\left< C_{\y \y} f, h\right>_{\H_\Rn} = \E[f(\y) h(\y)]$ for any $f, h \in \H_\Rn$, and $C_{\y \y} \otimes C_{\y \y}$ satisfies $\left< C_{\y \y} F_a, F_b\right>_{\H_\Rn \otimes \H_\Rn} = \E[ F_a(\y, \y') F_b(\y, \y') ]$ for any $F_a, F_b \in \H_\Rn \otimes \H_\Rn$.

\begin{proof}
By the triangle inequality, 
\begin{eqnarray} 
&& \left\|  \hat{C}_{\vartheta \y} ( \hat{C}_{\y \y} + \varepsilon_m  I)^{-1} k_\Rn(\cdot,Y^n) - \mu_{\Theta | r^* } \right\|_{\H_\Theta} \nonumber  \\
&\leq & \left\|  \hat{C}_{\vartheta \y} ( \hat{C}_{\y \y} + \varepsilon_m  I)^{-1} k_\Rn(\cdot,Y^n) - C_{\vartheta \y} ( C_{\y \y} + \varepsilon_m  I)^{-1} k_\Rn(\cdot,Y^n) \right\|_{\H_\Theta} \nonumber \\
&&+   \left\|  C_{\vartheta \y} ( C_{\y \y} + \varepsilon_m  I)^{-1} k_\Rn(\cdot,Y^n) - \mu_{\Theta | r^* } \right\|_{\H_\Theta} \nonumber \\
&\leq & \left\|  \hat{C}_{\vartheta \y} ( \hat{C}_{\y \y} + \varepsilon_m  I)^{-1}  - C_{\vartheta \y} ( C_{\y \y} + \varepsilon_m  I)^{-1} \right\| \left\|  k_\Rn(\cdot,Y^n) \right\|_{\H_\Theta} \label{eq:upper-first} \\
&&+   \left\|  C_{\vartheta \y} ( C_{\y \y} + \varepsilon_m  I)^{-1} k_\Rn(\cdot,r^*) - \mu_{\Theta | r^* } \right\|_{\H_\Theta} \label{eq:upper-second},
\end{eqnarray}
where we used Theorem \ref{theo:cme-population-misspecified} in the last line. 
Below we derive convergence rates of the two terms \eqref{eq:upper-first}\eqref{eq:upper-second} separately, and then determine the decay schedule of $\varepsilon_m$ as $m \to \infty$ so that the two terms have the same rate.

{\bf The first term \eqref{eq:upper-first}.} 
We first have
\begin{eqnarray*}
&& \hat{C}_{\vartheta \y} ( \hat{C}_{\y \y} + \varepsilon_m  I)^{-1}  - C_{\vartheta \y} ( C_{\y \y} + \varepsilon_m  I)^{-1} \\
&=& \hat{C}_{\vartheta \y} ( \hat{C}_{\y \y} + \varepsilon_m  I)^{-1} - \hat{C}_{\vartheta \y} ( C_{\y \y} + \varepsilon_m  I)^{-1}  \\
&&+ \hat{C}_{\vartheta \y} ( C_{\y \y} + \varepsilon_m  I)^{-1}  - C_{\vartheta \y} ( C_{\y \y} + \varepsilon_m  I)^{-1} \\
&=& \hat{C}_{\vartheta \y} \left[  ( \hat{C}_{\y \y} + \varepsilon_m  I)^{-1} -  ( C_{\y \y} + \varepsilon_m  I)^{-1} \right] \\
&&+  (\hat{C}_{\vartheta \y} - C_{\vartheta \y}) ( C_{\y \y} + \varepsilon_m  I)^{-1} \\ 
&=& \hat{C}_{\vartheta \y}  ( \hat{C}_{\y \y} + \varepsilon_m  I)^{-1} ( C_{\y \y} - \hat{C}_{\y \y} )  ( C_{\y \y} + \varepsilon_m  I)^{-1}  \\
&&+  (\hat{C}_{\vartheta \y} - C_{\vartheta \y}) ( C_{\y \y} + \varepsilon_m  I)^{-1} ,
\end{eqnarray*}
where the last equality follows from the formula $A^{-1} - B^{-1} = A^{-1} (B-A) B^{-1}$ that holds for any invertible operators $A$ and $B$.
Note that $\hat{C}_{\vartheta \y} = \hat{C}_{\vartheta \vartheta}^{1/2} W_{\vartheta \y} \hat{C}_{\y \y}^{1/2}$ holds for some $W_{\Theta \F}: \H_\Rn \to \H_\Theta$ with $\| W_{\vartheta \y} \| \leq 1$ \citep[Theorem 1]{Bak73}. 
Using this, we have  
\begin{eqnarray*}
&& \left\| \hat{C}_{\vartheta \y} ( \hat{C}_{\y \y} + \varepsilon_m  I)^{-1}  - C_{\vartheta \y} ( C_{\y \y} + \varepsilon_m  I)^{-1} \right\| \\
&\leq& \left\| \hat{C}_{\vartheta \y}  ( \hat{C}_{\y \y} + \varepsilon_m  I)^{-1} ( C_{\y \y} - \hat{C}_{\y \y} )  ( C_{\y \y} + \varepsilon_m  I)^{-1}  \right\| \\
&&+ \left\| (\hat{C}_{\vartheta \y} - C_{\vartheta \y}) ( C_{\y \y} + \varepsilon_m  I)^{-1} \right\| \\
&=& \left\|\hat{C}_{\vartheta \vartheta}^{1/2} W_{\vartheta \y} \hat{C}_{\y \y}^{1/2} ( \hat{C}_{\y \y} + \varepsilon_m  I)^{-1} ( C_{\y \y} - \hat{C}_{\y \y} )  ( C_{\y \y} + \varepsilon_m  I)^{-1}  \right\| \\
&&+ \left\| (\hat{C}_{\vartheta \y} - C_{\vartheta \y}) ( C_{\y \y} + \varepsilon_m  I)^{-1} \right\| \\
&\leq& \left\|\hat{C}_{\vartheta \vartheta}^{1/2} \right\|    \varepsilon_m ^{-1/2} \left\| ( C_{\y \y} - \hat{C}_{\y \y} )  ( C_{\y \y} + \varepsilon_m  I)^{-1}  \right\| \\
&&+ \left\| (\hat{C}_{\vartheta \y} - C_{\vartheta \y}) ( C_{\y \y} + \varepsilon_m  I)^{-1} \right\| \\
&=& O_p\left(   \varepsilon_m^{-3/2} m^{-1/2}  + \sqrt{N(\varepsilon_m)} \varepsilon_m^{-1} m^{-1/2}\right) \quad (m \to \infty,\ \varepsilon_m \to 0),
\end{eqnarray*}
where the second inequality follows from $\| W_{\vartheta y} \| \leq 1$ and $\| \hat{C}_{\y \y}^{1/2} ( \hat{C}_{\y \y} + \varepsilon_m I)^{-1} \| \leq \varepsilon_m^{-1/2}$, and the last line from  \citet[Lemma 1.5.1]{Fuk15}; the quantity $N(\varepsilon)$ for any $\varepsilon > 0$  is defined by $N(\varepsilon) := \mathrm{Tr}[ C_{\y \y} (C_{\y \y} +\varepsilon I)^{-1} ]$, where ${\rm Tr}(A)$ denotes the trace of an operator $A$.   
Under our assumption on the eigenvalue decay rate of $C_{\y \y}$, we have $N(\varepsilon) \leq \frac{\beta b}{b-1} \varepsilon^{-1/b}$ \citep[Proposition 3]{CapDev07}, which implies that the above rate becomes 
$$
O_p\left(   \varepsilon_m^{-3/2} m^{-1/2}  + \varepsilon_m^{-1 - 1/2b} m^{-1/2}\right) \quad (m \to \infty,\ \varepsilon_m \to 0).
$$
From $m \varepsilon_m \to \infty$ and $\varepsilon_m \to 0$ (as we determine the schedule of $\varepsilon_m$ below), it is easy to show that the second term is slower and thus dominates the above rate.
This concludes that the rate of the first term \eqref{eq:upper-first} is
$$
 \left\|  \hat{C}_{\vartheta \y} ( \hat{C}_{\y \y} + \varepsilon_m  I)^{-1}  - C_{\vartheta \y} ( C_{\y \y} + \varepsilon_m  I)^{-1} \right\| \left\|  k_\Rn(\cdot,Y^n) \right\|_{\H_\Theta} = O_p\left( \varepsilon_m^{-1 - 1/2b} m^{-1/2}\right) \quad (m \to \infty,\ \varepsilon_m \to 0). 
$$

{\bf The second term \eqref{eq:upper-second}.}
Let $(\vartheta', \y')$ be an independent copy of  the random variables $(\vartheta, \y)$.
Note that for any $\psi \in \H_\Rn$ , we have 
\begin{eqnarray*}
\left<  C_{\vartheta \y} \psi, C_{\vartheta \y} \psi \right>_{\H_\Theta}
&=& \E \left[ k_\Theta(\vartheta, \vartheta') \psi(\y) \psi(\y') \right] \\
&=& \E \left[ \E[ k_\Theta(\vartheta, \vartheta') | \y, \y'] \psi(\y) \psi(\y') \right] \\
&=& \E \left[  G(\y, \y') \psi(\y) \psi(\y') \right] \\
&=& \left< (C_{\y \y} \otimes C_{\y \y}) G, \psi \otimes \psi \right>_{\H_\Rn \otimes \H_\Rn}.
\end{eqnarray*}
Similarly, for any $\psi \in \H_\Rn$ and $\tilde{Y}^n \in \supp(P_\Rn)$, we have 
\begin{eqnarray*}
 \left< C_{\vartheta \y} \psi, \E [k_{\Theta} (\cdot, \vartheta)  | \y = \tilde{Y}^n ] \right>_{\H_\Theta} 
 &=& \E\left[ \psi(\y') \E [k_{\Theta} (\vartheta', \vartheta)  | \y = \tilde{Y}^n ]  \right]  \\
  &=& \E\left[  \psi(\y') \E [k_{\Theta} (\vartheta', \vartheta)  | \y = \tilde{Y}^n, \y' ]   \right]  \\
  &=& \E\left[ \psi(\y')  G( \tilde{Y}^n, \y')  \right]  \\
&=& \left< (I \otimes C_{\y \y}) G, k_\Rn(\cdot, \tilde{Y}^n) \otimes \psi \right>_{\H_\Rn \otimes \H_\Rn},
\end{eqnarray*}
where $I :\H_\Rn \to \H_\Rn$ is the identity operator and  
$$ 
\left( (I \otimes C_{\y \y}) G \right) (\cdot, *) := \E[G(\cdot, \y' ) k_\Rn(\y', *) ].
$$

Now let $\psi := (C_{\y \y} + \varepsilon_m I)^{-1} k_\Rn(\cdot, r^*)$. 
Recall $ \mu_{\Theta | r^* }  = \E[k_\Theta(\cdot, \vartheta) | \y = r^*]$, which gives $\| \mu_{\Theta | r^*}\|_{\H_\Theta}^2 = G(r^*, r^*)$.
Then the square of \eqref{eq:upper-second} can be written as 

\begin{eqnarray*}
 && \left\|  C_{\vartheta \y} ( C_{\y \y} + \varepsilon_m  I)^{-1} k_\Rn(\cdot,r^*) - \mu_{\Theta | r^* } \right\|_{\H_\Theta}^2 \\
 &=& \left\|   C_{\vartheta \y} \psi \right\|_{\H_\Theta}^2  - 2 \left<  C_{\vartheta \y} \psi,   \mu_{\Theta | r^* }  \right>_{\H_\Theta} + \| \mu_{\Theta | r^*}\|_{\H_\Theta}^2 \\ 
 &=&   \left< (C_{\y \y} \otimes C_{\y \y}) G,\ \ (C_{\y \y} + \varepsilon_m I)^{-1} k_\Rn(\cdot, r^*) \otimes (C_{\y \y} + \varepsilon_m I)^{-1} k_\Rn(\cdot, r^*) \right>_{\H_\Rn \otimes \H_\Rn}  \\
 && - 2  \left< (I \otimes C_{\y \y}) G,\ \  k_\Rn(\cdot,  r^*) \otimes  (C_{\y \y} + \varepsilon_m I)^{-1} k_\Rn(\cdot, r^*) \right>_{\H_\Rn \otimes \H_\Rn} + G(r^*, r^*) \\
  &=&   \left< ( (C_{\y \y} + \varepsilon_m I)^{-1} C_{\y \y}  \otimes   (C_{\y \y} + \varepsilon_m I)^{-1} C_{\y \y} ) G,\ \ k_\Rn(\cdot, r^*) \otimes k_\Rn(\cdot, r^*) \right>_{\H_\Rn \otimes \H_\Rn}  \\
 && - 2  \left< (I \otimes   (C_{\y \y} + \varepsilon_m I)^{-1} C_{\y \y}) G,\ \  k_\Rn(\cdot,  r^*) \otimes  k_\Rn(\cdot, r^*) \right>_{\H_\Rn \otimes \H_\Rn} + G(r^*, r^*) \\
 &=& \Big\langle   \Big\{ (C_{\y \y} + \varepsilon_m I)^{-1} C_{\y \y} \otimes (C_{\y \y} + \varepsilon_m I)^{-1} C_{\y \y}  - I \otimes (C_{\y \y} + \varepsilon_m)^{-1} C_{\y \y} \\ 
 && - (C_{\y \y} + \varepsilon_m I)^{-1} C_{\y \y} \otimes I + I \otimes I \Big \} G,\quad k_\Rn (\cdot,r^*) \otimes k_\Rn(\cdot,r^*) \Big\rangle_{\H_\Rn \otimes \H_\Rn} \\
 &\leq & \Big\|   \Big\{ (C_{\y \y} + \varepsilon_m I)^{-1} C_{\y \y} \otimes (C_{\y \y} + \varepsilon_m I)^{-1} C_{\y \y}  - I \otimes (C_{\y \y} + \varepsilon_m)^{-1} C_{\y \y} \\ 
 && - (C_{\y \y} + \varepsilon_m I)^{-1} C_{\y \y} \otimes I + I \otimes I \Big \} G \Big\|_{\H_\Rn \otimes \H_\Rn}\  \Big\| k_\Rn (\cdot,r^*) \otimes k_\Rn(\cdot,r^*) \Big\|_{\H_\Rn \otimes \H_\Rn}.
\end{eqnarray*}
Let $(\phi_i)_{i=1}^\infty \subset \H_\Rn$ be the eigenfunctions of $C_{\y \y}$ and $(\lambda_i)_{i=1}^\infty$ be the associated eigenvalues such that $\lambda_1 \geq \lambda_2 \geq \dots \geq 0$. 
Then the eigenfunctions and eigenvalues of the operator $C_{\y \y} \otimes C_{\y \y}$ are given as $(\phi_i \otimes \phi_j)_{i,j=1}^\infty$ and $(\lambda_i \lambda_i)_{i,j=1}^\infty$, respectively.  
Note that $(C_{\y \y} + \varepsilon_m I)^{-1}C_{\y \y}^2 \phi_i = (\frac{\lambda_i^2}{1+\varepsilon_m}) \phi_i$.
Note also that our assumption $G \in \mathrm{Range}(C_{\y \y} \otimes C_{\y \y})$ implies that there exists some $\xi \in \H_\Rn \otimes \H_\Rn$ such that $G = ( C_{\y \y} \otimes C_{\y \y}) \xi$.
Using these identities and Parseval's identity, we have
\begin{eqnarray*}
 && \Big\|   \Big\{ (C_{\y \y} + \varepsilon_m I)^{-1} C_{\y \y} \otimes (C_{\y \y} + \varepsilon_m I)^{-1} C_{\y \y}  - I \otimes (C_{\y \y} + \varepsilon_m)^{-1} C_{\y \y} \\ 
 && - (C_{\y \y} + \varepsilon_m I)^{-1} C_{\y \y} + I \otimes I \Big \} G \Big\|_{\H_\Rn \otimes \H_\Rn}^2 \\
&=&  \Big\|   \Big\{ (C_{\y \y} + \varepsilon_m I)^{-1} C_{\y \y} \otimes (C_{\y \y} + \varepsilon_m I)^{-1} C_{\y \y}  - I \otimes (C_{\y \y} + \varepsilon_m)^{-1} C_{\y \y} \\ 
 && - (C_{\y \y} + \varepsilon_m I)^{-1} C_{\y \y} + I \otimes I \Big \} (C_{\y \y} \otimes C_{\y \y}) \xi \Big\|_{\H_\Rn \otimes \H_\Rn}^2 \\ 
 &=& \sum_{i,j} \left\{ \frac{ \lambda_i^2 }{\lambda_i + \varepsilon_m} \frac{\lambda_j^2}{\lambda_j + \varepsilon_m} - \frac{\lambda_i \lambda_j^2}{\lambda_j + \varepsilon_m} - \frac{\lambda_i^2 \lambda_j}{\lambda_i + \varepsilon_m} + \lambda_i \lambda_j \right\}^2 \left< \phi_i \otimes \phi_j, \xi \right>_{\H_\Rn \otimes \H_\Rn}^2  \\
 &=& \sum_{i,j} \left\{  \frac{\varepsilon_m^2 \lambda_i \lambda_j}{(\lambda_i + \varepsilon_m) (\lambda_j + \varepsilon_m) } \right\}^2 \left< \phi_i \otimes \phi_j, \xi \right>_{\H_\Rn \otimes \H_\Rn}^2 \\
 &\leq & \varepsilon_m^4 \| \xi \|_{\H_\Rn \otimes \H_\Rn}^2.
\end{eqnarray*}
From this the second term \eqref{eq:upper-second} is upper-bounded as
\begin{eqnarray*}
&& \left\|  C_{\vartheta \y} ( C_{\y \y} + \varepsilon_m  I)^{-1} k_\Rn(\cdot,r^*) - \mu_{\Theta | r^* } \right\|_{\H_\Theta} \\
&\leq&   \varepsilon_m \| \xi \|_{\H_\Rn \otimes \H_\Rn}^{1/2} \Big\| k_\Rn (\cdot,r^*) \otimes k_\Rn(\cdot,r^*) \Big\|_{\H_\Rn \otimes \H_\Rn}^{1/2} = O(\varepsilon_m), \quad (m \to \infty,\ \varepsilon_m \to 0).
\end{eqnarray*}
The obtained rates for the two terms  \eqref{eq:upper-first}\eqref{eq:upper-second} can be balanced by setting $\varepsilon_m = C m^{- \frac{b}{1+4b}}$ for any fixed constant $C > 0$, and this gives the rate in the assertion.
\end{proof}

\section{Experiments on Sophisticated Production Simulator}


\begin{figure*}[t]
  \centering
  \includegraphics[width=12cm]{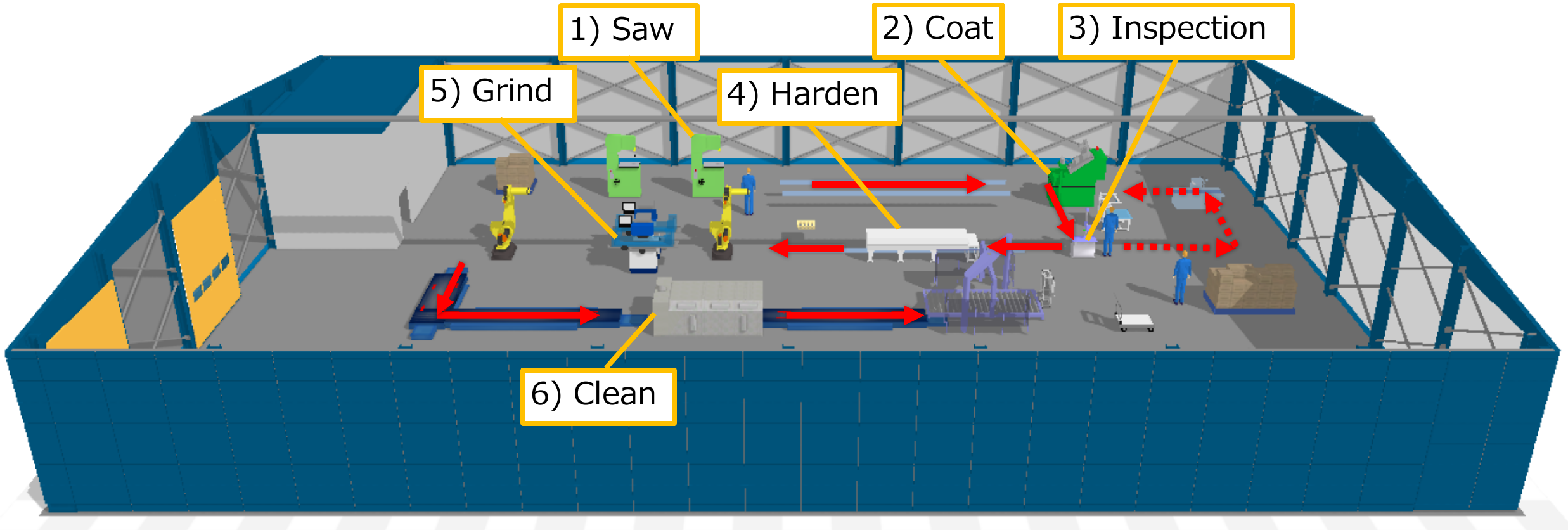}
  \caption{Illustration of the manufacturing process  (metal processing factory) for producing valves.}
  \label{fig:ACME_overview_3D}
\end{figure*}

\begin{table*}[t]
\caption{Summary of the true and estimated parameters for the experiment on the sophisticated simulation model. $T_{\rm BF}$ represents the mean time between failures, and $T_R$ the mode of repair time for each process.  The parameter estimates are the posterior means of the generated parameters, averaged over 10 independent trials, and the corresponding standard deviations are shown in brackets.}
\begin{center}
\begin{tabular}{c|cc|cc|cc|cc|cc|cc}
Process  & \multicolumn{2}{c|}{Saw} & \multicolumn{2}{c|}{Coat}  & \multicolumn{2}{c|}{Inspection} & \multicolumn{2}{c|}{Harden} & \multicolumn{2}{c|}{Grind} & \multicolumn{2}{c}{Clean} \\ \hline
& $T_{\rm BF}$ & $T_{\rm R}$ & $T_{\rm BF}$ & $T_{\rm R}$ & $T_{\rm BF}$ & $T_{\rm R}$ & $T_{\rm BF}$ & $T_{\rm R}$ & $T_{\rm BF}$ & $T_{\rm R}$ & $T_{\rm BF}$ & $T_{\rm R}$ \\
Parameters & $\theta_1$ & $\theta_2$ & $\theta_3$ & $\theta_4$ & $\theta_5$ & $\theta_6$ & $\theta_7$ & $\theta_8$ & $\theta_9$ & $\theta_{10}$ & $\theta_{11}$ & $\theta_{12}$ \\ \hline
true $\theta^{(0)}$ ($x<140$) & 100 & 25 & 200 & 10 & 70 & 20 & 200 & 20 & 75 & 15 & 120 & 20 \\
true $\theta^{(1)}$ ($x>140$) & 100 & 25 & 200 & 10 & {\bf 50} & 20 & 200 & 20 & 75 & 15 & 120 & 20 \\ \hline
posterior mean  & 104.6 & 25.3 & 181.2 & 7.1 &  70.9 & 18.9 & 180.1 & 18.9 & 72.5 & 15.2 & 121.7 &  20.2\\ [-1.pt]
for ordinary reg. & \small (4.4) & \small (1.2) & \small (7.9) & \small (0.3) & \small (7.6) & \small (0.8) & \small (8.4) & \small (0.3) & \small (3.9) & \small (0.9) & \small (5.1) & \small (1.2)\\ [1pt]
posterior mean & 99.4 & 25.4 & 181.2 & 7.9 & {\bf 54.5} & 22.1 & 176.4 & 17.9 & 75.6 & 14.9 & 120.6 & 20.4 \\ [-1.pt]
for covariate shift & \small (6.1) & \small (0.9) & \small (7.5) & \small (0.1) & \small (6.2) & \small (2.2) & \small (4.4) & \small (0.1) & \small (3.6) & \small (0.5) & \small (5.1) & \small 0.7
\end{tabular}
\label{table:ACME_parameters}
\end{center}
\end{table*}

\label{subsubsec:ResultOfRealisticExperiment}
\begin{figure}[t]
  \centering
  \includegraphics[width=12cm]{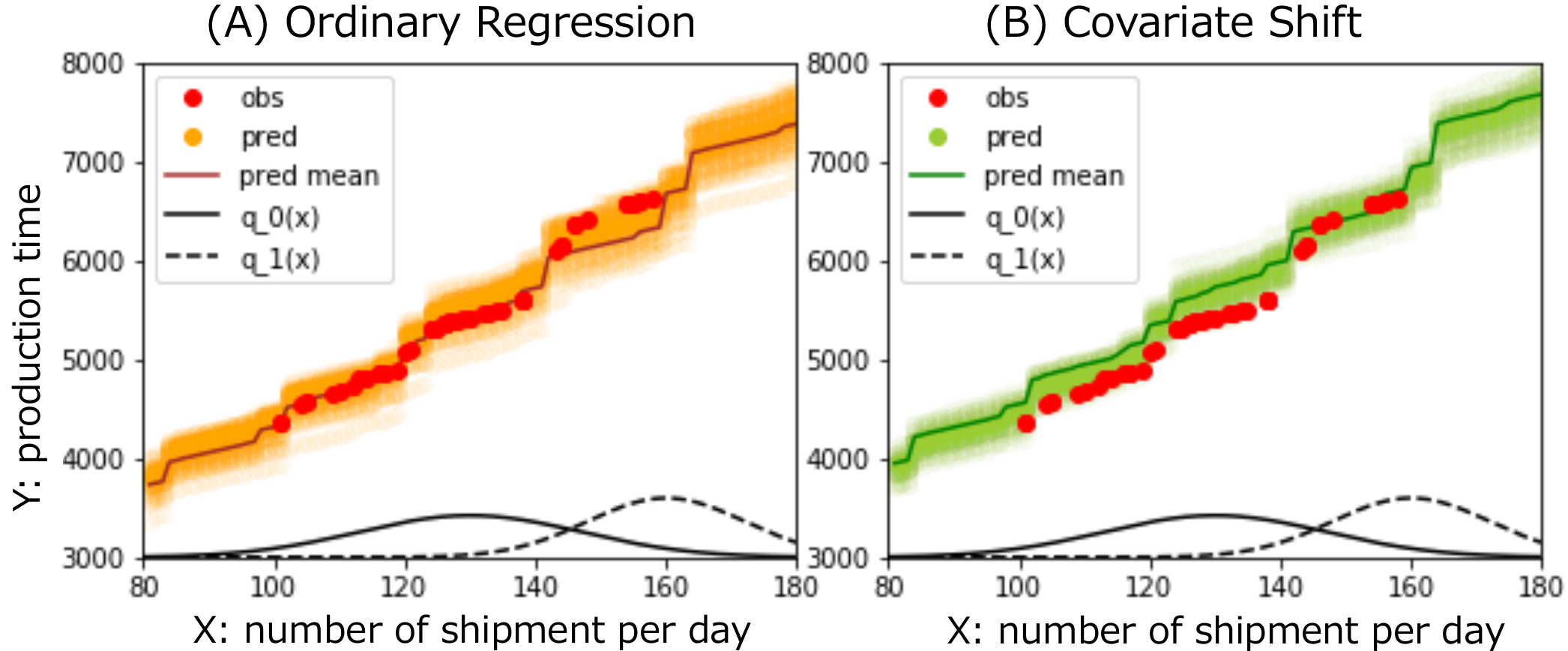}
  \caption{Results of ordinary regression and covariate shift adaptation, for the experiment on the sophisticated model.
    (A) Results of our method {\em without} covariate shift adaptation: training data (red points), generated predictive outputs (orange) and their means (brown curve).
  (B) Results of our method {\em with} covariate shift adaptation: training data (red points), generated predictive outputs (light green) and their means (green curve). 
  $q_0(x)$ and $q_1(x)$ are input densities for training and prediction, respectively. }
  \label{fig:result_ACME}
\end{figure}

We performed experiments on a sophisticated but more complicated simulator for industrial manufacturing processes than the one in Sec.~\ref{subsec:ExperimentProductSimulator}.
We used a simulation model constructed with the software package WITNESS (\url{https://www.lanner.com/en-us/}) described in Fig.~\ref{fig:ACME_overview_3D}.
It models a metal processing factory for producing valves (products) from metal pipes, with six primary processes of  1) ``saw'', 2) ``coat'', 3) ``inspection'', 4) ``harden'', 5) ``grind'', and  6) ``clean.'' 
Each process consists of complicated procedures such preparation, waiting, and machine repair in case of a trouble.

\subsection{Setting}

As in Sec.~\ref{subsec:ExperimentProductSimulator},  the input space is $\X = (0, \infty)$ and each input $x$ represents the number of products required to make, and the resulting output $y(x) = R(x) + e(x)$ is the length of time needed to produce that number of products.

The mapping $x \to r(x, \theta)$ consists of the above six processes, and each of them contains two parameters for machine downtime due to failures: the mean time between failures ($T_{\rm BF}$), and the mode of repair time ($T_{\rm R}$).
Thus, in total, there are 12 parameters, i.e., $\theta = ( \theta_1,..., \theta_{12} )^\top \in \Theta \subset \R^{12}$, where $\theta_{2j} = T_{\rm BF}^{(j)}$ and $\theta_{2j + 1} = T_{\rm R}^{(j)}$ for the $j\ (= 1,\dots,6$)-th process (see Table~\ref{table:ACME_parameters}). 
In each process (say the $j$-th process), the time between two failures follows the negative exponential distribution with the mean time $\theta_{2j} = T_{\rm BF}^{(j)}$, and the time required for repair follows the Erlang distribution with the mode of repair time  $\theta_{2j + 1} = T_{\rm R}^{(j)}$  and the shape parameter 3.
We set the prior distribution  $\pi(\theta)$ by defining the uniform distribution over [0, 300] for  $\theta_{2j}$  and that over [0,30] for $\theta_{2j + 1}$, and taking the product of the uniform distributions for all the parameters ($j = 1, \dots, 6)$.

In a similar manner to the experiment in Section~\ref{subsec:ExperimentProductSimulator}, we defined the regression function $R(x)$ of the data generating process as $R(x) = r(x,\theta^{(0)})$ for $x < 140$ and  $R(x) = r(x,\theta^{(1)})$ for $x \geq 140$, where $\theta^{(0)}$ and $\theta^{(1)}$ are the ``true'' parameters for training and prediction, and defined in Table~\ref{table:ACME_parameters}.
We set the input densities $q_0(x)$ and $q_1(x)$ for training and prediction as $N(130, 15)$ and $N(160, 12)$, respectively.
The size of training data is $n=50$, and the number of simulations is $m=400$.
We set the noise process of the data generating process to be independent Gaussian, $e(x)=\epsilon \sim N(0, 300)$.
We set the constants $\sigma^2, \sigma_\Theta^2  > 0$ in the kernels $k_\Rn$ and $k_\Theta$ by the median heuristic using the simulated pairs $(\bar{\theta}_j, \bar{Y}^n_j)_{j=1}^m$, and the regularization constant to be $\varepsilon=0.1$.

\subsubsection{Details of the Simulation Model}
We explain below qualitative details of the six processes in the simulation model constructed with the WITNESS software package. 

{\bf Cutting process: }
The manufacturing process begins with the arrival of pipes, all of which have the same diameter and length of 30 cm.
These pipes arrive at a fixed time interval, depending on the vendor's supply schedule.
Subsequently, each pipe is cut into 10-cm sections along the length, resulting in three pieces.
A worker is assigned for this process to perform changeover, repair, and disconnection operations. 
This worker takes a break once every eight hours.
Then the small pieces obtained are transferred to the coating process by a conveyor belt.

{\bf Coating process: }
The small pieces are coated for protection by a coating machine. 
The machine processes six pieces in a batch manner at once.
A coating material must have been prepared in the coating machine, before those pieces have arrived; otherwise, the quality of those pieces will be degraded by heat.
When the pieces ride on the belt conveyor, a sensor detects them and the coating material is prepared.

{\bf Inspection process:}
After the coating process, each piece is placed in an inspection waiting buffer.
An inspector picks up those pieces one by one from the waiting buffer, and inspects the coating quality.
If a piece fails the quality inspection, the inspector places that piece in the recoating waiting buffer.
The coating machine must process the pieces of the recoating buffer preferentially.
When pieces pass the quality inspection, the inspector sends those pieces to the curing step.

{\bf Harden process: }
In the harden (quenching) process, up to 10 pieces are processed simultaneously in a first-come first-out basis, and each piece is quenched for at least one hour.

{\bf Grind process: }
The quenched pieces are polished to satisfy a customer's specifications.
Two polishing machines with the same priority are available. 
Each machine uses special jigs to process four pieces simultaneously, and produces two different types of valves. 
Further, 10 jigs exist in the system, and when not in use, they are placed in a jig storage buffer.

A loader fixes four pieces with a jig and sends it to the polishing machine.
The polishing machine sends the jig and the four pieces to an unloader, once polishing is done.
The unloader sends the finished pieces to a valve storage area and the jig to a jig return area.
The two types of valves are separated, and placed in a dedicated valve storage buffer.
When a jig is required to be used again, it is returned by a jig return conveyor to the jig storage buffer.

{\bf Cleaning process: }
Valves issued from a valve storage area are cleaned before shipment.
In the washing machine, five stations are available where valves can be placed one at a time, and the valves are cleaned in these stations.
Up to 10 valves of each type can be washed simultaneously.
When the valve type is changed, the cleaning head must be replaced.

\subsection{Results }

The true 12 parameters are estimated as the posterior means of generated parameters, and their averages and standard deviations over 10 independent trials are shown in the bottom rows in Table~\ref{table:ACME_parameters}.
Most of the true parameters are estimated for both of the ordinary regression and covariate shift settings.
%

Fig.~\ref{fig:result_ACME}~(A) and (B) describe predictive outputs and their means given by the proposed method, which fit well for both the ordinary and covariate shift settings.
%
The RMSE for predictive outputs by the proposed method with covariate shift adaptation, calculated for test data generated from $q_1(x)$, is $1.48 \times 10^2$.   
On the other hand, the RMSE on the same test data for the proposed method {\em without} covariate shift adaptation (i.e., setting $\beta(X_i) = 1, i =1,\dots,n$ in the importance-weighted kernel) is $1.64 \times 10^3$. 
This confirms that the use of the importance-weighted kernel indeed works for covariate shift adaptation.

In this experiment, approximately 3~[s] was required for one evaluation of the simulation model $r(x,\theta)$ with the authors' computational environment. 
Thus, the dominant factor in the computational cost was that of simulations.

\end{document}